\documentclass{article}
\usepackage{a4wide}
\usepackage{amsthm}
\usepackage{setspace}
\usepackage{amssymb}
\usepackage{amsfonts}       % blackboard math symbols
\usepackage{booktabs}       % professional-quality tables
\usepackage{nicefrac}       % compact symbols for 1/2, etc.
\usepackage{microtype}      % microtypography
\usepackage[dvipsnames]{xcolor}         % colors
\usepackage[utf8]{inputenc} % allow utf-8 input
\usepackage[T1]{fontenc}    % use 8-bit T1 fonts
\usepackage{hyperref}       % hyperlinks
\usepackage{url}            % simple URL typesetting

\usepackage{natbib}
\usepackage[pdftex]{graphicx}
\usepackage{subcaption}
\usepackage{array,multirow}
\usepackage{enumitem}
\usepackage{bm, mathtools, amsmath}
\usepackage{amsthm}
\usepackage[capitalize,noabbrev]{cleveref}
\usepackage[dvipsnames]{xcolor}
\usepackage{selectp}
\theoremstyle{plain}
\newtheorem{theorem}{Theorem}[section]

\newtheorem{lemma}[theorem]{Lemma}

\theoremstyle{definition}

\theoremstyle{remark}

\usepackage{booktabs}
\usepackage{wrapfig}

\usepackage[ruled, lined, linesnumbered, commentsnumbered, longend]{algorithm2e}
\usepackage{algcompatible}

\RestyleAlgo{ruled}
\SetKwComment{Comment}{/*}{*/}

\usepackage[textsize=tiny]{todonotes}

\DeclareMathOperator{\clip}{clip}

\title{Mitigating Relative Over-Generalization in Multi-Agent Reinforcement Learning}

\author{Ting Zhu$^\mathbf{1}$, Yue Jin$^\mathbf{2}$, Jeremie Houssineau$^\mathbf{3}$, Giovanni Montana$^\mathbf{1,2,4}$\\ 
$^1$Department of Statistics, University of Warwick, Coventry, UK,\\ 
$^2$Warwick Manufacturing Group, University of Warwick, Coventry, UK,\\ 
$^3$School of Physical \& Mathematical Sciences, Nanyang Technological University, Singapore,\\ 
$^4$Alan Turing Institute, London, UK\\ 
\texttt{\{ting.zhu, yue.jin.3, g.montana\}@warwick.ac.uk, jeremie.houssineau@ntu.edu.sg}}

\begin{document}
\maketitle

\begin{abstract}
In decentralized multi-agent reinforcement learning, agents learning in isolation can lead to relative over-generalization (RO), where optimal joint actions are undervalued in favor of suboptimal ones. This hinders effective coordination in cooperative tasks, as agents tend to choose actions that are individually rational but collectively suboptimal. To address this issue, we introduce MaxMax Q-Learning (MMQ), which employs an iterative process of sampling and evaluating potential next states, selecting those with maximal Q-values for learning. This approach refines approximations of ideal state transitions, aligning more closely with the optimal joint policy of collaborating agents. We provide theoretical analysis supporting MMQ's potential and present empirical evaluations across various environments susceptible to RO. Our results demonstrate that MMQ frequently outperforms existing baselines, exhibiting enhanced convergence and sample efficiency.  
\end{abstract}

\section{Introduction}

Cooperative multi-agent reinforcement learning (MARL) has become increasingly important for addressing complex real-world challenges that require coordinated behaviors among multiple agents. Successful applications included playing card games \citep{Brown2018Superhuman}, autonomous driving \citep{shalev2016safe, pmlr-v155-zhou21a}, unmanned aerial vehicles \citep{wang2020multi}, wireless sensor networks \citep{xu2020learningnetworks, sahraoui2021schedule} and traffic light control \citep{bazzan2009opportunitiescontrol, zhou2023novel}. A dominant framework in MARL is centralized training with decentralized execution (CTDE) \citep{lowe2017multi,2018RashidQMIX,son2019qtran}, which often relies on a centralized coordinator to aggregate information, such as local observation or individual actions from other agents during training. While this approach has been widely adopted due to its effectiveness in leveraging global information, it may still face scalability challenges, particularly in environments with large number of agents or complex interactions. Additionally, in scenarios where privacy is a critical concern, CTDE methods that require access to individual agent data could be less desirable. While many CTDE methods do not require exhaustive local information from all agents, the potential for scalability and privacy issues in certain implementations warrants consideration. While inter-agent communication can partially mitigate some of these challenges \citep{foerster2016learning, zhu2022surveycommunication}, it introduces additional overhead, which can be prohibitive in environments where communication is costly or unreliable.

%which employs a centralized coordinator to access and collect local information (i.e., actions and rewards) from all agents at each update step. However, this approach may encounter scalability and privacy issues, limiting its practicality in complex real-world settings. While inter-agent communication can partially mitigate these challenges \citep{foerster2016learning, zhu2022surveycommunication}, it introduces significant overhead, making it impractical in environments where communication is costly or unreliable.

%Consider, for instance, an open multi-agent system like mixed autonomous traffic, comprising both human drivers and autonomous vehicles \citep{valiente2022robustness}. The fluctuating number and diversity of agents in such a system make it challenging to consistently gather local information from all participants or ensure reliable communication. 

Consider large-scale drone swarms for search and rescue or surveillance tasks \citep{baldazo2019decentralized_drone, batra2022decentralized_swarms}. In a CTDE framework, centrally aggregating data from each drone can be impractical due to network delays, dynamic environments, and potential communication failures. Additionally, communication frameworks may struggle with channel congestion or high packet loss, especially in complex terrain or with electronic interference. Fully decentralized learning presents a promising alternative, where agents rely solely on their experiences without considering the actions or policies of other agents during both training and execution. This approach enables the system to scale effectively and remain resilient to communication issues. However, decentralized approaches come with their own challenges. 
From an individual agent's perspective, the learning process occurs within a non-stationary MDP, as the transition probabilities change due to the evolving policies of other agents.
%Agents must coordinate actions with limited knowledge of others, \textcolor{red}{and the non-stationarity issue arises as the transition probabilities, from an individual agent's perspective, change due to the evolving policies of other agents.}\textcolor{blue}{[response to Reviewer sjZT about non-stationarity]}

%exacerbating issues such as non-stationarity and uncertainty about joint action effects. 

Existing decentralized MARL methods, including optimism strategies in Q-learning \citep{lauer2000algorithm, matignon2007hysteretic, wei2016lenient}, attempt to mitigate these challenges. More recently, Ideal Independent Q-learning (I2Q) \citep{jiang2022i2q} explicitly models state transitions assuming optimal joint behavior, introducing ideal transition probabilities to address non-stationarity in independent Q-learning.

Another critical issue in decentralised MARL is \emph{Relative Over-generalisation} (RO), where agents prefer suboptimal policies because individual actions appear preferable without coordinated strategies. First defined and studied within the field of cooperative co-evolutionary algorithms \citep{wiegand2004analysis, panait2006biasing, panait2007analysis}, it has more recently been explored predominantly in centralised MARL contexts \citep{rashid2020weighted, gupta2021uneven, shi2022curriculum}. RO occurs when agents adapt their actions to limited interactions, often focusing on the exploratory behaviours of other agents. Consequently, agents may favour more robust but less optimal solutions in the absence of coordinated strategies. In decentralised learning settings, this problem has been discussed within simple matrix game scenarios \citep{wei2016lenient}. RO, exacerbated by non-stationarity, presents a significant challenge as agents make decisions based on fluctuating global rewards without the benefit of coordinated strategies \citep{matignon2012independent, wei2016lenient}. Although previous implementations of optimistic strategies have shown some efficacy, our empirical results indicate they fall short in cooperative tasks with pronounced RO challenges.

%Another critical issue in decentralized MARL is \textcolor{red}{Relative Over-generalization (RO), where agents prefer suboptimal policies because individual actions seem preferable in the absence of coordinated strategies. This pathology is first defined and analyzed in \cite{wiegand2004analysis}. The "generalization" here means that agents over-fit their actions to exploratory behaviors of other agents based on limited interactions. As a result, it may inherently favor robust rather optimal solutions.} 

This paper introduces MaxMax Q-Learning (MMQ), a novel algorithm designed to address the RO problem in decentralised MARL settings. MMQ aims to mitigate the challenges posed by RO and non-stationarity in decentralised learning environments. The key insight behind MMQ is enabling agents to reason about beneficial experiences that occur infrequently. At its core, MMQ employs two non-parameterised quantile models to capture the range of state transitions, accounting for both environmental factors and the evolving policies of learning agents. These models iteratively sample, evaluate, and select optimal states, refining the approximation of ideal state transitions and facilitating global reward maximisation. The state with the highest Q-value is then selected to update the value function, promoting convergence towards optimal Q-values in the context of other agents' best actions. The MMQ algorithm incorporates two maximum operators in the Bellman update: the first takes the maximum over all possible next states to select the most promising future scenario, and the second takes the maximum over Q-values of state-action pairs to determine the best action in that scenario. MMQ's adaptive nature, which involves continuously updating the ranges of possible next states, enables effective decision-making in dynamic environments.

The main contributions of this paper are threefold. First, we introduce MMQ, a novel algorithm that employs quantile models to capture multi-agent dynamics and approximate ideal state transitions through sampling. Second, we provide a theoretical demonstration of MMQ's potential to converge to globally optimal joint policies, assuming perfect knowledge of forward and value functions. Third, we present empirical results showing that MMQ often outperforms or matches established baselines across various cooperative tasks, highlighting its potential for faster convergence, enhanced sample efficiency, and improved reward maximisation in decentralised learning environments.

The remainder of this paper is structured as follows: Section 2 discusses related work in MARL and uncertainty quantification. Section 3 provides background on multi-agent Markov Decision Processes and the challenges of relative over-generalization. Section 4 presents the methodology of MaxMax Q-Learning, including its theoretical foundations and implementation details. Section 5 describes our experimental setup and results across various environments. Finally, Section 6 concludes the paper with a discussion of our findings and potential directions for future research.

\section{Related work}

\textbf{Centralised learning methods.} Within the centralized training paradigm, RO has been discussed mostly for value factorization methods like QMIX \citep{2018RashidQMIX}. The monotonic factorization in QMIX cannot represent the dependency of one agent's value on others' policies, making it prone to RO. Proposed solutions include weighting schemes during learning \citep{rashid2020weighted}, curriculum transfer from simpler tasks  \citep{gupta2021uneven, shi2022curriculum}, and sequential execution policy \citep{liu2024solving}. Soft Q-learning extensions to multi-agent actor-critics \citep{wei2018multiagentSoft, lowe2017multi} utilise energy policies for global search to mitigate RO. However, unlike our decentralized approach, these methods require a centralized critic with joint action access during training.

\textbf{Fully decentralized learning.} Decentralized approaches in MARL aim to overcome the scalability and privacy issues associated with centralized methods. However, they face unique challenges, particularly in addressing non-stationarity and RO. Existing approaches can be categorized based on their strategies for tackling these issues.
Basic independent learning methods like Independent Q-learning (IQL) \citep{tan1993multi} and independent PPO \citep{yu2022surprising} form the foundation of decentralized MARL. However, their simultaneous updates can lead to non-stationarity, potentially compromising convergence. Recent work by \citet{su2023fully} on DPO addresses this by providing monotonic improvement and convergence guarantees. To mitigate negative impacts of uncoordinated learning, several methods promote optimism toward other agents' behaviors. Distributed Q-learning \citep{lauer2000algorithm} selectively updates Q-values based only on positive TD errors. Hysteretic Q-learning \citep{matignon2007hysteretic} uses asymmetric learning rates, while Lenient Q-learning \citep{wei2016lenient} selectively ignores negative TD errors. These techniques aim to overcome convergence to suboptimal joint actions by dampening unhelpful Q-value changes. Taking a different approach, the recently introduced Ideal Independent Q-learning (I2Q) \citep{jiang2022i2q} explicitly models ideal cooperative transitions. However, it requires learning an additional utility function over state pairs. Our proposed method, MMQ, builds upon these approaches by encoding uncertainty about decentralized MARL dynamics. We model other agents as sources of heteroscedastic uncertainty with an epistemic flavor, providing a more flexible way to represent optimistic policies. By sampling from possible next states, MMQ avoids the need for heuristic corrections or separate Q-functions, offering a novel solution to the challenges of decentralized MARL.

\textbf{Uncertainty quantification.} Quantifying different sources of uncertainty is crucial in reinforcement learning, particularly in multi-agent settings. Prior work distinguishes between aleatoric uncertainty, arising from environment stochasticity, and epistemic uncertainty, due to insufficient experiences \citep{osband2016deep, depeweg2016learning}. Various methods, including variance networks \citep{kendall2017uncertainties, wu2021uncertainty} and ensembles \citep{lakshminarayanan2017simple}, have been proposed to model these uncertainties, with applications in single-agent RL \citep{chua2018deepmodels, sekar2020planningmodels}.
MARL introduces additional complexity due to the dynamic nature of agent interactions, leading to non-stationarity. This non-stationarity limits an agent's ability to reduce epistemic uncertainty through repeated state visits \citep{hernandez2017survey} and can be viewed as another form of epistemic uncertainty. Our proposed MMQ algorithm addresses these challenges by using quantile networks to effectively manage two key sources of epistemic uncertainty in multi-agent settings: limited experiential data and evolving strategies of other agents. This approach allows MMQ to better handle the unique uncertainties present in decentralized MARL environments.

\section{Background and preliminaries}

\subsection{Multi-agent Markov Decision Process}

Consider a multi-agent Markov Decision Process (MDP) represented by $\mathcal{M} = (\mathcal{S}, \mathcal{A}, R, P_{\mathrm{env}},\gamma)$. Within this tuple, $\mathcal{S}$ denotes the state space, $\mathcal{A}$ is the joint action space, $P_{\mathrm{env}}(s'|s,\bm{a})$ and $R(s,s')$ are respectively the environment dynamics and reward function for states $s,s' \in \mathcal{S}$ and action $\bm{a} \in \mathcal{A}$, and $\gamma$ is the discount factor. Given $N$ agents, the action space is of the form $\mathcal{A} = \mathcal{A}_1 \times \dots \times \mathcal{A}_N$ with any action $\bm{a} \in \mathcal{A}$ taking the form $\bm{a} = (a_1,\dots,a_N)$. At each time step $t$, an agent indexed by $i \in \{1,\dots,N\}$ selects an individual action, $a_i$. When the $N$ actions are executed, the environment transitions from state $s$ to state $s'$, and every agent receives a global reward, $r_t$. The objective is to maximize the expected return, i.e., $\mathbb{E}[\sum_{t=0}^T \gamma^t r_t]$, where $T$ is the time horizon. The individual environment dynamics is defined as 
$$
P_{i}(s'|s, a_i) = \sum_{\bm{a}_{-i}} P_{\mathrm{env}}(s'|s,\bm{a}) \pi_{-i}(\bm{a}_{-i}|s)
$$
where $\bm{a}_{-i}$ represents the joint action excluding agent $i$ and $\pi_{-i}$ is the joint policy of all other agents. Here, the joint action $\bm{a}$ inherently depends on $a_i$ and $\bm{a}_{-i}$. From any individual agent's perspective, the learning process occurs within a non-stationary MDP due to the evolving policy $\pi_{-i}$.

\subsection{Relative over-generalization through an example}

\begin{table}[t]
% \vspace{-10pt}
  \caption{Payoff matrix for a two-agent game} 
 \centering
 \begin{tabular}{cc|ccc|}
 & & \multicolumn{3}{c|}{Agent 2} \\
 & & A & B & C \\
 \hline
 \parbox[t]{2mm}{\multirow{3}{*}{\rotatebox[origin=c]{90}{Agent 1}}}
 & A & +3 & -6 & -6\\
 & B & -6 & 0 & 0\\
 & C & -6 & 0 & 0 \\
 \hline
 \end{tabular}
 \label{fig:matrix}
\end{table}

Consider a two-agent game with the reward structure shown in Table~\ref{fig:matrix}. In this game, there are three possible actions: $A, B$, and $C$. Agents would receive a joint reward of $+3$ if they take $A$ together. However, if only one agent takes $A$, that agent incurs a penalty of $-6$. Agents end up choosing less optimal actions ($B$ or $C$) if they perceive the reward for choosing $A$ to be lower, based on their expectations of the other agent's actions. For agent $1$, the utility function $Q_1(\cdot)$ is related to the probability that the other agent chooses $A$, i.e. $\pi_2(A)$. In this case, $Q_1(A)$ would be smaller than $Q_1(B)$ or $Q_1(C)$ if $\pi_2(A)<\frac{2}{5}$. This threshold arises because the expected value of choosing $A$ becomes lower than choosing $B$ or $C$ when the probability of the other agent also choosing $A$ falls below $\frac{2}{5}$. Consequently, during initial uniform exploration where $\pi_1(A)=\pi_2(A)=\frac{1}{3}$, both agents tend to favour $B$ or $C$ over $A$, even though $A$ is the globally optimal choice. Thus, with independent learning without considering the other agent's best action, both agents may end up with choosing suboptimal actions and fail to cooperate.

 \subsection{Independent Q-Learning}

In independent Q-learning \citep{tan1993multi}, each agent $i$ learns a policy independently, treating other agents as part of the environment. The individual Q-function is $Q_i(s, a_i) = \mathbb{E}\left[\sum_{t=0}^\infty \gamma^t r_t \mid s_0=s, a_{i,0} = a_i\right]$ for agent $i$. Each agent updates its Q-function by minimizing the loss $\mathbb{E}_{P_{i}(s'|s, a_i)} \left[(y_i - Q_i(s, a_i))^2\right]$, where $y_i$ is the target value defined as $R(s, s') + \gamma \max_{a_i'}Q_i(s', a_i')$. The RO problem arises in this setup as each agent seeks to maximize its own expected return based on experiences where other agents' policies evolve and contain random explorations. 

\subsection{Ideal transition probabilities}
 
To address the RO problem in this context, some approaches introduce implicit coordination mechanisms centered on the concept of an \emph{ideal transition model} \citep{lauer2000algorithm, matignon2007hysteretic, wei2016lenient, palmer2018negative, jiang2022i2q}. 
These methods guide each agent's learning with hypothetical transitions that assume optimal joint behavior, aligning independent learners towards coordination. Let $\boldsymbol{\pi}_{-i}$ denote the joint policy of other agents, and $Q^*$ the optimal joint Q-function. The optimal joint policy of other agents can be expressed as $\pi_{-i}^*(s, a_i) = \arg\max_{\bm{a}_{-i}} Q^*(s, a_i, \bm{a}_{-i})$.

The concept of ideal transitions refers to hypothetical state transitions that assume other agents are following optimal joint policies. These ideal transition probabilities represent the dynamics that would occur if all agents achieved perfect coordination, and are defined as $P_{i}^*(s'|s, a_i) = P_{\mathrm{env}}(s'|s, a_i, \pi_{-i}^*(s, a_i))$. Based on these probabilities, the Bellman optimality equation is given by
\begin{equation} \label{eq:bellman_original}
Q_i^*(s, a_i) = \mathbb{E}_{P_{i}^*(s'|s, a_i)} \left[R(s, s') + \gamma \max_{a_i'}Q_i^*(s', a_i')\right],
\end{equation}
where $Q_i^*(s, a_i)$ is the optimal Q-function for agent $i$.

An important theoretical result from \citet{jiang2022i2q} establishes that when all agents perform Q-learning based on these ideal transition probabilities, the individual and joint optimality align, that is, $\max_{a_i}Q_i^*(s, a_i) = Q^*(s, \pi^*(s))$, where $Q^*(s, \bm{a})$ is the optimal joint Q-function for any action $\bm{a} \in \mathcal{A}$, and $\pi^*$ is the optimal joint policy. However, achieving true ideal transitions is intractable in practice due to the evolving, uncontrolled nature of learning agents. This motivates developing techniques to approximate ideal transitions. 

\section{MaxMax Q-learning Methodology}

\subsection{Approximation of Bellman optimality equation}
\label{sec:mmq}

Our methodology aims to approximate ideal transition probabilities, which assume other agents follow optimal joint policies. We focus on deterministic environments and reformulate the Bellman optimality in Eq. \eqref{eq:bellman_original} to highlight the dependence on the set $\mathcal{S}_{s,a_i}$ of possible next states,
\begin{equation}
\label{eq:setNextStates}
\mathcal{S}_{s, a_i} = \left\{s' = f_{\mathrm{env}}(s, a_i, \bm{a}_{-i}) \mid \bm{a}_{-i} \in \prod_{j \neq i} \mathcal{A}_j \right\},
\end{equation}
where $f_{\mathrm{env}}(s, a_i, \bm{a}_{-i})$ is the deterministic transition function that maps the current state $s$ and the joint actions of all agents $(a_i, \bm{a}_{-i})$ to the next state $s'$, $\mathcal{A}_j$ is the action space for each agent $j$ and the Cartesian product $\prod_{j \neq i} \mathcal{A}_j$ represents all possible combinations of actions by the other agents. It is noted that we only use $f_{\mathrm{env}}$ to define the set $\mathcal{S}_{s,a_i}$ here, but do not need to learn this global transition function directly in our algorithm.  Encoding the deterministic transitions by a delta function, $\delta_{f_{\mathrm{env}}(s,\bm{a})} (s')$, Eq. \eqref{eq:bellman_original} is rewritten as
\begin{subequations}\label{bellman_new}
\begin{align}
Q_i^*(s, a_i) & = \mathbb{E}_{s' \sim \delta_{f_{\mathrm{env}}(s, a_i, \pi_{-i}^*(s, a_i))}(s')} \left[R(s, s') + \gamma \max_{a_i'} Q_i^*(s', a_i')\right] \\
\label{bellman_new_2}
& = \max_{s' \in \mathcal{S}^*_{s, a_i}} \left(R(s, s') + \gamma \max_{a_i'} Q_i^*(s', a_i')\right).
\end{align}
\end{subequations}
where $\mathcal{S}^*_{s, a_i}$ is any subset of $\mathcal{S}_{s, a_i}$ including $s^{\prime*} = f_{\mathrm{env}}(s, a_i, \pi_{-i}^*(s, a_i))$. 

%This reformulation allows us to target the optimal value $Q^*(s,a_i)$ under true coordinated joint behavior without directly approximating $s^{\prime*}$, which is challenging due to the evolving nature of other agents' policies.

By reformulating the Q-value optimization over the set $\mathcal{S}^*_{s,a_i}$, our approach allows for targeting the optimal value $Q^*(s,a_i)$ under the true coordinated joint behavior, without directly approximating $s^{\prime*}$. When there is no information about $s^{\prime*}$, the set $\mathcal{S}^*_{s, a_i}$ could be set in principle to $\mathcal{S}_{s, a_i}$, if it were known, to ensure the inclusion of $s^{\prime*}$. However, this will also make the maximisation over $\mathcal{S}^*_{s, a_i}$ in Eq. \eqref{bellman_new_2} more computationally challenging, implying a trade off between reducing the size of $\mathcal{S}^*_{s, a_i}$ and ensuring the inclusion of $s^{\prime*}$. We propose a learning procedure that enables each agent to progressively shrink their set of next states $\mathcal{S}^*_{s,a_i}$, as all agents explore and accumulate experiences.

In practice, at the algorithmic step $t$, we work with a subset $\hat{\mathcal{S}}_{s, a_i, t}$ which approximates one of the possible subsets $\mathcal{S}^*_{s, a_i}$. Since neither $\mathcal{S}_{s, a_i}$ nor $s^{\prime*}$ are known in practice due to the incomplete information about other agents' policies and the environment dynamics, we cannot guarantee that $s^{\prime*} \in \hat{\mathcal{S}}_{s, a_i, t} \subseteq \mathcal{S}_{s, a_i}$ holds, but we will show in our performance assessment that $s^{\prime*} \in \hat{\mathcal{S}}_{s, a_i, t}$ holds with high probability. Furthermore, direct maximization over $\hat{\mathcal{S}}_{s, a_i, t}$ is challenging as this set is infinite in general. 

To address this, we resort to Monte Carlo optimization, as in e.g.\ \citet{robert1999monte}, by introducing a finite set $\hat{\mathcal{S}}^M_{s, a_i, t}$ of $M$ points randomly sampled from $\hat{\mathcal{S}}_{s, a_i, t}$. Assuming no approximation error in the predicted bound, that is $\hat{\mathcal{S}}_{s, a_i, t} \subseteq \mathcal{S}_{s, a_i}$ holds, the sample set $\hat{\mathcal{S}}^M_{s, a_i, t}$ contains only reachable states and it follows that
\begin{subequations} \label{eq:max-based_Qvalue}
\begin{align}
Q_i^*(s, a_i) & \geq \max_{s' \in \hat{\mathcal{S}}^M_{s, a_i, t}} \left[ R(s, s') + \gamma \max_{a_i'} Q_i^*(s', a_i')\right] \\
\label{eq:max-based_Qvalue_2}
& = \!\max_{m \in \{1, \dots, M\}}\! \left[R(s, s'_m) + \gamma \max_{a'_i} Q_i^*(s'_m, a'_i)\right].
\end{align}
\end{subequations}
With the considered approach, there are two natural phases when running the associated algorithms:
\begin{enumerate}[nosep]
\item With little information to rely on, the agents explore the state space at random and collect diverse trajectories, which improve their understanding of the range of possible next state $\mathcal{S}_{s, a_i}$. In this phase, the estimated sets $\hat{\mathcal{S}}_{s, a_i, t}$ will be close to $\mathcal{S}_{s, a_i}$.
\item As agents refine their estimates of the set $\mathcal{S}_{s, a_i}$ of possible next states and accumulate sufficient reward information, the maximisation in Eq. \eqref{eq:max-based_Qvalue_2} yields increasingly stable values, facilitating policy convergence. This, in turns, means that the new trajectories will be more similar and optimised, progressively outnumbering the initial diverse trajectories. This will cause the sets $\hat{\mathcal{S}}_{s, a_i, t}$ to zero in on $s^{\prime*}$, hence facilitating the Monte Carlo optimisation in Eq. \eqref{eq:max-based_Qvalue_2}.
\end{enumerate}

\begin{figure}[t]
  \centering
   % \vspace{-10pt}
  \includegraphics[width=0.45\linewidth]{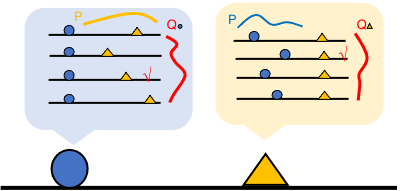}
  \caption{Illustration of the MMQ update for two agents. 
Different positions of two agents in the rounded rectangles represent different possible next states $s'=(x_{b},x_{y})$. From the perspective of the blue agent: The yellow curve (\textcolor{YellowOrange}{\textbf{P}}) represents the distribution of states with different yellow agent positions ($x_y$) in the replay buffer. The red curve (\textcolor{red}{\textbf{Q}}) represents the estimated Q-values for those possible next states. In the MMQ update, the blue agent selects samples for update based on the highest Q-value, marked by \textcolor{red}{$\surd$}. Importantly, this selection may not always coincide with the most frequently encountered scenarios (corresponding to the peak of the yellow curve) from past experiences, which may be sub-optimal.
}

  \label{fig:ill}
\end{figure}

% \begin{figure}[!ht]
%          \centering
%      \begin{subfigure}[b]{0.55\textwidth}
%          \includegraphics[width=\textwidth]{Figures/illustration.png}
%          \caption{}
%          \label{fig:ill}
%      \end{subfigure}
%      \begin{subfigure}[b]{0.37\textwidth}
%          \centering
%          \includegraphics[width=\textwidth]{Styles/Figures/state_set.png}
%          \caption{}
%          \label{fig:state_set}
%      \end{subfigure} \\
%     \caption{(a) Illustration of the MMQ Update from the Perspective of the Blue Agent: The distribution of the Yellow Agent's position, shown by the yellow curve (\textcolor{YellowOrange}{\textbf{P}}), is derived from the replay buffer. The red curve (\textcolor{red}{\textbf{Q}}) represents the estimated Q-values for different states. In this scenario, the Blue Agent selects samples based on the highest Q-value, marked by \textcolor{red}{$\surd$}. Importantly, this selection may not always coincide with the most frequently encountered scenarios from past experiences, which may be the sub-optimal experience; (b) Illustration of the set relationship. }
% \end{figure}

An illustration of our sampling and selection process is shown in Figure~\ref{fig:ill}:  given a set of possible next state samples, our algorithm selects the state with the highest estimated Q-value for updating, which implicitly indicating the optimal action of other agents. As agents explore more possible actions, the estimated set $\hat{\mathcal{S}}_{s, a_i, t}$ increasingly approximates the true set $\mathcal{S}_{s, a_i}$. Crucially, if the optimal next state is contained within the estimated set, the equality 
$$ Q_i^*(s, a_i) = \max_{s' \in \hat{\mathcal{S}}_{s, a_i, t}} \left[R(s, s') + \gamma \max_{a_i'} Q_i^*(s', a_i')\right] $$
holds. This property is fundamental to our method, as it implies that through iterative learning and effective sampling, each agent can learn Q-values that closely align with those derived from ideal transition probabilities. In the following section, we analyse the convergence properties of this approach under ideal conditions. The complete algorithm, including implementation details, will be presented in Section~\ref{algorithm}.

\subsection{Convergence analysis} \label{sec:convergece}

Our combined learning and sampling approach facilitates the gradual convergence of the agents' policies toward the globally-optimal joint policy. This gradual convergence is supported by the insights from the following theorem, which shows the disparity between the optimal Q-values and those learned by the agents is limited by the difference between the best next state in the estimated set and the true best next state. This bounding relationship is crucial, as it indicates that the closer our estimated set of next states is to the actual set composed of all the possible next states, the more accurate the estimations of the agents' optimal Q-functions become.

In this section, we further elaborate on this mechanism and provide a formal convergence analysis. This analysis demonstrates how our proposed model-based Q-learning approach, combined with the sampling strategy, effectively facilitates convergence to an optimal global policy. We begin by showing that the difference between the Q-values learned by our approach and the optimal Q-values depends on how well we can estimate the best next state.

\begin{theorem}\label{conv}
Let $\mathcal{S}_{s, a_i}$ be the set of all possible next states as defined in \eqref{eq:setNextStates} and let \(\hat{S}\) be a non-empty subset of $\mathcal{S}_{s, a_i}$. Let \(s'^{*}\) and \(\hat{s}'^*\) represent the best next states in the optimal and approximate regimes, respectively, that is
\begin{align*}
    s'^{*} & = \arg\max_{s' \in \mathcal{S}_{s, a_i}} R(s,s') + \gamma \max_{a'_i} Q_i^*(s',a'_i) \\
    \hat{s}'^{*} & = \arg\max_{s' \in \hat{S}} R(s,s') + \gamma \max_{a'_i} Q_i(s',a'_i).
\end{align*}
Under Assumptions~\ref{as:Rlip}-\ref{as:order2} (see Appendix~\ref{sec:proofs}), if the Euclidean distance \(d(s'^{*}, \hat{s}'^*)\) is at most \(\epsilon\) for all \((s, a_i)\), then there exists $K > 0$ such that $|Q_i^*(s, a_i) - Q_i(s, a_i)| \leq (1-\gamma)^{-1}K\epsilon$ for all $(s, a_i)$.
\end{theorem}

The proof can be found in Appendix~\ref{sec:proofs}. 

%In the context of interest, omitting the algorithm, step $t$ from the notations, \(\hat{S}\) is a set $\hat{\mathcal{S}}^M_{s,a_i}$, consisting of \(M\) states uniformly sampled from $\hat{\mathcal{S}}_{s, a_i}$, the non-empty subset of $\mathcal{S}_{s, a_i}$ formed by all possible outcomes from the learned model (see the illustration of those sets in Figure~\ref{fig:state_set}).
In the context of our algorithm, we can relate this theorem to our specific implementation. Omitting the algorithm step $t$ from the notations for simplicity, $\hat{S}$ in our case corresponds to $\hat{\mathcal{S}}^M_{s,a_i}$. This set consists of $M$ states uniformly sampled from $\hat{\mathcal{S}}{s, a_i}$, which is itself a non-empty subset of $\mathcal{S}{s, a_i}$. $\hat{\mathcal{S}}_{s, a_i}$ is formed by all possible outcomes predicted by our learned model. Figure~\ref{fig:state_set} illustrates the relationships between these sets. 

This result demonstrates that if the distance between the estimated best next state \(\hat{s}'^*\) and the actual best next state \(s'^{*}\) is arbitrarily small for all state-action pairs \((s, a_i)\), then the discrepancy between the learned Q-values \(Q_i(s, a_i)\) and the optimal Q-values \(Q_i^*(s, a_i)\) is bounded. This implies that as we refine our estimation of the optimal next state through iterative sampling and learning, we progressively narrow the gap between the learned Q-values of our agents and the true optimal values. 

\begin{figure}[ht]
  \centering
  \includegraphics[width=0.3\linewidth]{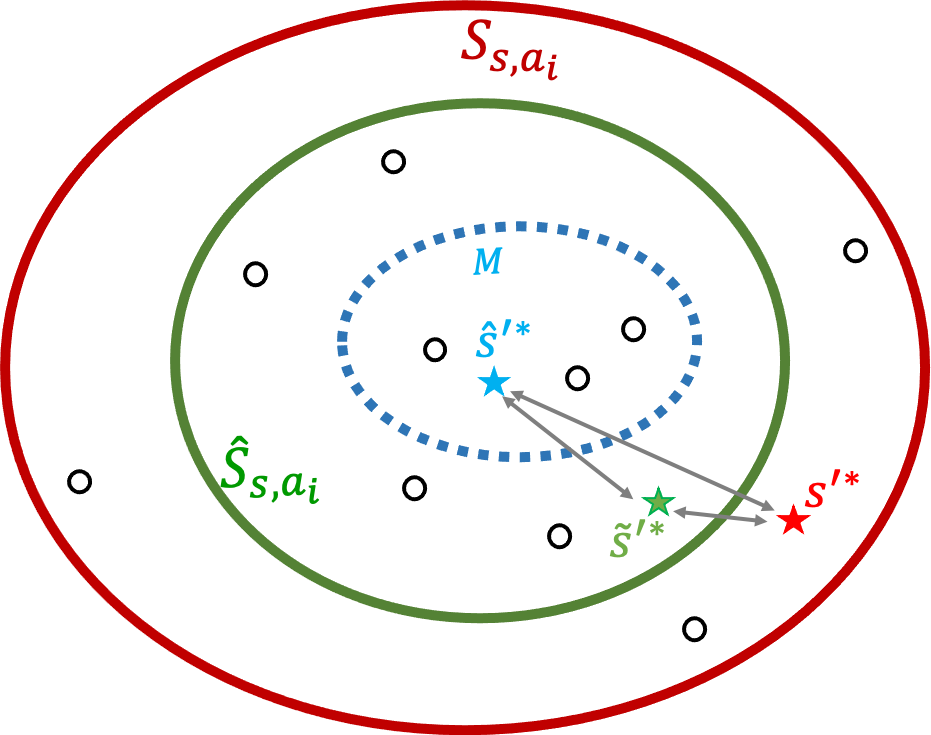}
  \caption{Illustration of the set relationship among $\mathcal{S}_{s,a_i}$, $\hat{\mathcal{S}}_{s,a_i}$ and $\hat{\mathcal{S}}^M_{s,a_i}$ (denoted as $M$ above). The red star, $s'^*$, is the best next state in the real set, and $\tilde{s}'^*$, and $\hat{s}'^*$ represent the two states that are the closest to the best next states in $\hat{\mathcal{S}}_{s,a_i}$ and $\hat{\mathcal{S}}^M_{s,a_i}$. According to the triangle inequality, the distance between \textcolor{red}{$s'^*$} and \textcolor{Cyan}{$\hat{s}'^*$}, $d$(\textcolor{red}{$s'^*$},\textcolor{Cyan}{$\hat{s}'^*$}), is upper bound by the sum of $d$(\textcolor{red}{$s'^*$},\textcolor{Green}{$\tilde{s}'^*$}) and $d$(\textcolor{Green}{$\tilde{s}'^*$},\textcolor{Cyan}{$\hat{s}'^*$})}
  \label{fig:state_set}
\end{figure}

To better analyse the distance $d(s', \hat{s}')$, we introduce a third state $\tilde{s}'$ as the best next state in $\hat{\mathcal{S}}_{s, a_i}$. This allows us to use the triangle inequality to upper bound the distance as ,
% \[ d(s', \hat{s}') \leq d(s', \tilde{s}') + d(\tilde{s}', \hat{s}'). \]
$d(s', \hat{s}') \leq d(s', \tilde{s}') + d(\tilde{s}', \hat{s}')$

The first term, $d(s', \tilde{s}')$, reflects the difference between the optimal next state in $\mathcal{S}_{s, a_i}$ and the one in $\hat{\mathcal{S}}_{s, a_i}$. As the size of $\hat{\mathcal{S}}_{s, a_i}$ increases in the first phase of the algorithm, it's more likely to include states closer to $s'$, thus decreasing $d(s', \tilde{s}')$. The second term, $d(\tilde{s}', \hat{s}')$, represents the error in the Monte Carlo optimization. This error tends to be large initially but decreases in the second phase as $\hat{\mathcal{S}}_{s, a_i}$ shrinks, making it easier to sample states close to $\tilde{s}'$. This analysis shows how our algorithm progressively improves its estimation of the optimal next state, contributing to the overall convergence of the Q-values.

\begin{theorem}\label{distance}
Assume that $\mathcal{S} = \mathbb{R}$ and that $\hat{\mathcal{S}}_{s,a_i}$ is of the form $[-u, u]$ for some $u \in (0, \infty)$. Consider \(M\) i.i.d.\ samples \(s'_1, \dots, s'_M\) from  the uniform distribution on $[-u, u]$. It holds that
\[
\mathbb{E}\bigg[ \min_{k = 1, \dots, M} | \tilde{s}^{\prime *} - s'_k | \bigg] < \frac{2u}{M + 1}.
\]
\end{theorem}
The proof can be found in Appendix~\ref{sec:proofs}. 

This result demonstrates that the Monte Carlo optimization error \( d(\tilde{s}^{\prime *}, \hat{s}'^*) \) diminishes as the number of samples \( M \) increases.  %\textcolor{red}{Incidentally, there could be an exponential dependence between as dimensionality of state increases. However, in multi-agent scenarios, we typically consider partially-observed MDPs, which would limit the growth of dimension}. 
Incidentally, the Monte Carlo optimisation error could exhibit exponential dependence as the dimensionality of the state space increases. However, in multi-agent scenarios, partially-observed MDPs typically limit effective dimensionality growth, as agents rely on a restricted view of the overall state space.

Seemingly, there is an inherent trade-off involved in expanding the state set \( \hat{\mathcal{S}}_{s, a_i} \), i.e., expanding $u$ in the above one-dimension case, to cover more possibilities while managing the resultant error. A broader \( \hat{\mathcal{S}}_{s, a_i} \) reduces the gap between \(\tilde{s}^{\prime *}\) and the true best state \( s'^{*} \), as it increases the likelihood of encompassing \(s'^* \). This action effectively shrinks the error term \( d(s'^{*}, \tilde{s}^{\prime *}) \). Yet, increasing the size of $\hat{\mathcal{S}}_{s,a_i}$, i.e., increasing \(u\), also typically increases variability, leading to a larger error \( d(\tilde{s}^{\prime *}, \hat{s}'^*) \) and necessitating more samples to maintain a given level of precision.

\subsection{Implementation details}
\label{algorithm}

\begin{algorithm}[!ht] 
\caption{MMQ for each agent $i$}
\label{pseudocode_h}
\begin{algorithmic}
    \STATE {\bfseries Input:} Q-network $Q_i$, actor network $\pi_i$ and target networks $\overline{Q}_i$, $\overline{\pi}_i$; Quantile models $g^{\tau_l}_i$ and $g^{\tau_u}_i$; Reward network $R_i$; Replay buffer $D_i$
    \FOR{$t=1,\dots,T_{\mathrm{max}}$}
    \STATE All agents interact with the environment using random action (for initial $P$ steps) or $\varepsilon$-greedy policy and store experiences $(s,a_i,r,s')$ in $D_i$
    \STATE Sample a mini-batch from $\mathcal{D}_i$; 
    \STATE Update $g^{\tau_l}_i$ and $g^{\tau_u}_i$ by minimizing \eqref{eq:Quan_loss} 
    \STATE Draw $M$ samples from predicted bound $[g_i^{\tau_l}(\cdot), g_i^{\tau_u}(\cdot)]$ to construct set $\hat{\mathcal{S}}$ 
    \STATE Calculate the target value over $\hat{\mathcal{S}}$ using \eqref{eq:target}
    \STATE Update $Q_i$ multiple times by minimizing \eqref{eq:Q_loss}
    \STATE Update $R_i$ by minimizing \eqref{eq:R_loss}
    \STATE Update $\pi_i$ by minimizing \eqref{eq:A_loss}
    \STATE Update the target network $\overline{Q}_i$ and $\overline{\pi}_i$ 
    \ENDFOR
\end{algorithmic}
\end{algorithm}

To capture the range of possible next states, our implementation utilises two non-parametrised quantile models, \( g^{\tau_l}_i \) and \( g^{\tau_u}_i \), which employ neural networks to predict the \( \tau_l = 0.05 \) and \( \tau_u = 0.95 \) quantiles for each dimension of the next state. The neural network parameters, denoted by \( \phi^l_i \) and \( \phi^u_i \), are learnt by minimising the quantile loss according to their respective \( \tau \) values over samples from individual replay buffer $D_i$:
\begin{equation}
    L(\phi_i) = \mathbb{E}_{s,a_i \sim \mathcal{D}_i}[L^{\tau}(g_i^{\tau}(s, a_i; \phi_i) - s')],
    \label{eq:Quan_loss}
\end{equation}
where \( L^{\tau}(u) = \mathbb{I}(u > 0) \tau u + \mathbb{I}(u < 0) (1 - \tau) u \). For each \( (s, a_i) \) pair, the two quantile models predict bounds \([g_i^{\tau_l}(s, a_i), g_i^{\tau_u}(s, a_i)]\). We then construct the potential next state set \( \hat{\mathcal{S}} \) by including the true \( s' \) and \( M \) samples drawn from the quantile bounds. We also explored a parametrised multivariate Gaussian model as another method to estimate the possible next states, detailed in Appendix \ref{Quantile_Gau}.

Each agent also learns a Q-network, parameterised by \( \theta_i \), and represented as \( Q_i(s, a_i; \theta_i) \). The target value for the Q-network update is given by:
\begin{equation} 
Y_i(s, a_i; \theta_i) = \max_{\hat{s}' \in \hat{\mathcal{S}}} \left( R_i(s, s'; \psi_i) + \gamma \max_{a'_i} Q_i(\hat{s}', a'_i; \theta_i) \right), 
\label{eq:target}
\end{equation}
where \( R_i(s, s'; \psi_i) \) is an estimate from a learned reward model parameterised by \( \psi_i \). Here we apply a stop-gradient operator to the target value to prevent gradient flow back to the next state estimation process. This operation is crucial for maintaining stability in the learning dynamics by separating the optimization of the Q-network from the updates to the next state estimation. The loss function for optimising the Q-network parameter \( \theta_i \) is then given by
\begin{equation}
L(\theta_i) = \mathbb{E}_{s,a_i\sim \mathcal{D}_i}\left[ Q_i(s, a_i; \theta_i) - Y_i(s, a_i; \theta_i) \right]^2.
\label{eq:Q_loss}
\end{equation}

This loss aims to align the Q-network's predictions with the maximum expected return, considering both the immediate reward and the discounted future Q-values of the potential next states sampled from the estimated quantile bound. Additionally, the learned reward function \( R_i(s, s'; \psi_i) \), parameterised by \( \psi_i \), is trained to approximate the rewards for state transitions. The loss for this reward model is the mean squared error between the predicted and actual rewards,
\begin{equation}
L(\psi_i) = \mathbb{E}_{s\sim\mathcal{D}_i}[R_i(s, s'; \psi_i) - r]^2.
\label{eq:R_loss}
\end{equation}

To deal with continuous action spaces, each agent uses an actor network \( \pi_i(s; \rho_i) \), parameterised by \( \rho_i \), to learn its policy. The loss function for the actor network, aimed at maximising the Q-value, is 
\begin{equation}
L(\rho_i) = \mathbb{E}_{s\sim\mathcal{D}_i}[-Q_i(s, \pi_i(s; \rho_i); \theta_i)].
\label{eq:A_loss}
\end{equation}

The training process is summarised in Algorithm \ref{pseudocode_h}, and the full source code is available at \href{https://github.com/Tingz0/Maxmax\_Q\_learning}{\texttt{https://github.com/Tingz0/Maxmax\_Q\_learning}}. The full algorithm interleaves the optimisation of various constituent models, allowing agents to adaptively learn and improve their policies based on their own experiences and the evolving environmental dynamics. Our implementation incorporates two key strategies. First, a delayed update approach for the actor network relative to the critic network, where the critic is updated 10 times more frequently to maintain stability \citep{fujimoto2018addressing}. Second, negative reward shifting \citep{sun2022exploit}, which enhances our double-max-style updates (see also Appendix \ref{app:negative}).

\section{Experimental results}

\subsection{Environments} We evaluated the MMQ algorithm in three types of cooperative MARL environments characterized by the need for complex coordination among agents; see Figure~\ref{fig:all_env} for an overview.

% \vspace{-10pt}
\begin{figure}[!ht]
 % \vspace{-5pt}
         \centering
        \begin{minipage}[b]{0.38\textwidth}
        \centering
        \captionsetup{font=footnotesize}
        \begin{subfigure}[b]{0.51\textwidth}
         \includegraphics[width=\textwidth]{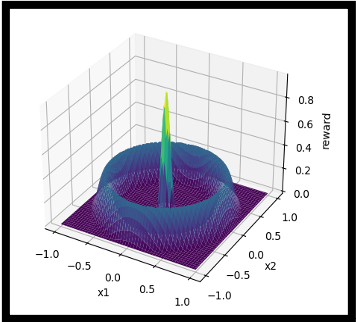}
         \caption{Differential Game (DG)}
         \label{fig:env_DG}
        \end{subfigure}
        \vfill
        \begin{subfigure}[b]{0.51\textwidth}
         \centering
         \includegraphics[width=\textwidth]{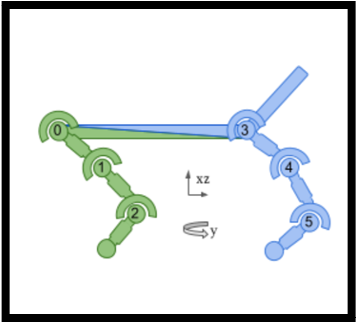}
         \caption{Half-Cheetah 2x3}
         \label{fig:env_Half}
     \end{subfigure}
     \end{minipage}
     \hspace{-0.05\textwidth}
     \begin{minipage}[b]{0.58\textwidth}
     \begin{subfigure}[b]{\textwidth}
         \centering
         \includegraphics[width=\textwidth]{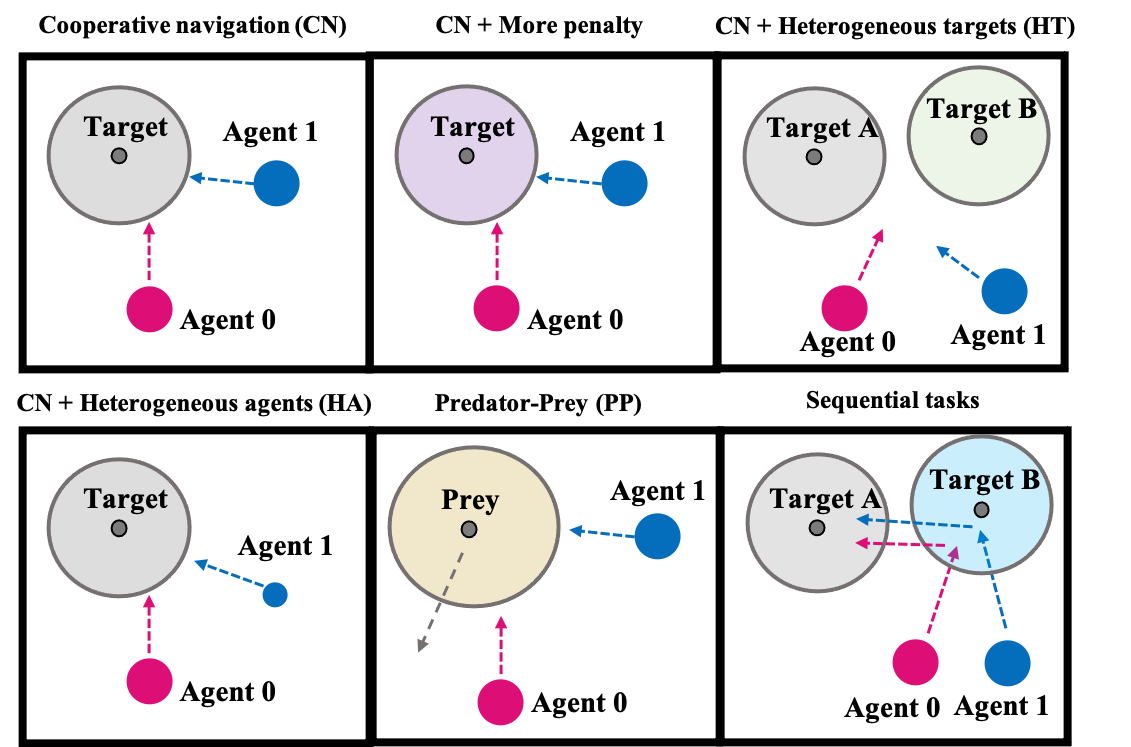}
         \caption{MPE scenarios}
         \label{fig:env_MPE}
     \end{subfigure} 
      \end{minipage}
    \caption{Task visualization. (a) \textit{Differential Game}(DG): agents need to cross a wide zero-reward area to move to the center to gain the optimal reward. (b) \textit{Half-Cheetah 2x3}: the Half-Cheetah 2x3 scenario in MAmujoco domain; (c) \textit{MPE scenarios}; \textit{Cooperative navigation}(CN): two agents need to enter the grey area of the target together to gain the reward, the solo entry would induce a penalty. \textit{CN + More penalty}: Same task as CN but with more penalty for solo entry; \textit{CN + HT}: Agents could choose to approach one of the two Targets with different reward settings; \textit{CN+HA}: same task as CN but two agents have different sizes and velocity; \textit{Predator-Prey(PP)}: two agents need to enter the grey area of a pre-trained prey. \textit{Sequential Task}: two agents need to first go through the grey area of Target B and then enter the grey of Target A with the same RO reward design as CN.}
    \label{fig:all_env}
\end{figure}

% \begin{figure}[!ht]
%   \centering
%   \includegraphics[width=0.7\linewidth]{Figures/MPE_env2.png}
%   \caption{Task visualization. \textit{Differential Game}(DG): agents need to cross a wide zero-reward area to move to the center to gain the optimal reward.\textit{Cooperative navigation}(CN): two agents need to enter the grey area of the target together to gain the reward, the solo entry would induce a penalty. \textit{CN + More penalty}: Same task as CN but with more penalty for solo entry; \textit{CN + HT}: Agents could choose to approach one of the two Targets with different reward settings; \textit{CN+HA}: same task as CN but two agents have different sizes and velocity; \textit{Predator-Prey(PP)}: two agents need to enter the grey area of a pre-trained prey. \textit{Sequential Task}: two agents need to first go through the grey area of Target B and then enter the grey of Target A with the same RO reward design as CN. \textit{Half-Cheetah 2x3} The Half-Cheetah 2x3 scenario in MAmujoco domain}
%   \label{fig:all_env}
% \end{figure} 

\paragraph{Differential Games} 

We adapted this environment from \cite{jiang2022i2q}, where $N$ agents move within the range $[-1,1]$. At each time step, an agent indexed by $i$ selects an action $a_i \in [-1,1]$. The state of this agent then transitions to $\operatorname{clip}\{x_i + 0.1 \times a_i, -1, 1\}$, where $x_i$ is the previous state and the $\operatorname{clip}(y, y_{\mathrm{min}}, y_{\mathrm{max}})$ function restricts $y$ within $[y_{\mathrm{min}}, y_{\mathrm{max}}]$. The global state is the position vector $(x_1, x_2)$. The reward function, detailed in Appendix~\ref{app:env_d}, assigns rewards following each action. A narrow optimal reward region is centred, surrounded by a wide zero-reward area and suboptimal rewards at the edges (see DG in Figure~\ref{fig:all_env}). This setup can lead to RO problems as agents might prefer staying in larger suboptimal areas.

\paragraph{Multiple Particle Environment} 
We designed six variants of cooperative navigation tasks with RO rewards as shown in Figure~\ref{fig:all_env}. The common goal is for two disk-shaped agents, \( D_1 \) and \( D_2 \), to simultaneously reach a disk-shaped target. To encourage coordination, we introduce a penalty for scenarios where only one agent is within a certain distance from the target. Specifically, we define a disk \( D \) centered on the target with radius \( r_D \) and penalize agents if only one is within \( D \). This setup is designed to illustrate the RO problem, where agents might prefer staying outside \( D \) rather than risk being the only one inside it. The task difficulty increases as the radius \( r_D \) decreases. The reward function, designed to reflect the RO problem, is defined as:
\[
r_{\mathrm{CN}} =
\begin{cases*}
R_{in} & if \( D_i \cap D \neq \emptyset \), \( i=1,2 \) \\
R_{out} & if \( (D_1 \cup D_2) \cap D = \emptyset \) \\
R_{out} - p & otherwise.
\end{cases*}
\]
The rewards for the three cases should satisfy \( R_{out} - p < R_{out} < R_{in} \). Detailed descriptions of \( R_{out} \) and \( R_{in} \) for different settings are provided in Appendix \ref{app:env_d}. In task \textbf{CN}, the penalty for solo entry into the circle is \( p=0.2 \); in task \textbf{CN+More Penalty}, the penalty increases to \( p=0.5 \) for entering the circle alone; in task \textbf{CN+Heterogeneous Agents (HA)}, two agents performing the \textbf{CN} task are heterogeneous, having different sizes and velocities; in task \textbf{CN+Heterogeneous Targets (HT)}, there are two targets, where entering the circle of target \( A \) follows the previous RO design, and entering the circle of target \( B \) incurs no RO penalty but offers a reward lower than \( R_{in } \); in the sub-optimal scenario, agents might only enter the circle of target \( B \); in task \textbf{Sequential Task}, agents must coordinate over a longer period—they could either directly reach target \( A \) with the same RO reward as before or first reach target \( B \) to pick up cargo, then receive a bonus each step (a higher \( R_{in } \)) when they later enter target \( A \) together; in task \textbf{Predator-Prey (PP)}, two predators (which we control) and one prey, who interact in an environment with two obstacles. The prey, trained using MADDPG \citep{lowe2017multi}, is adept at escaping faster than the predators. The predators need to enter the prey's disk together to receive the reward \( R_{in } \).

\paragraph{Multi-agent MuJoCo Environment} We employ the Half-Cheetah 2x3 scenario from the Multi-agent MuJoCo framework \citep{de2020deep}. This environment features two agents, each controlling three joints of the Half-Cheetah robot via torque application. It presents a partial observability setting, with each agent accessing only its local observations. We implement an RO reward structure designed to necessitate high coordination between agents. The reward function \( r_v \) is defined as: $-7$ if $v < v_l$, $-9$ if $v_l \leq v \leq v_u$, and $-2$, otherwise; 
% \[
% r_v = \begin{cases}
% -7 & \text{if } v < v_l \\
% -9 & \text{if } v_l \leq v \leq v_u \\
% -2 & \text{otherwise}
% \end{cases}
% \]
where \( v_l = 0.035/dt \), \( v_u = 0.04/dt \), and \( dt = 0.05 \). This structure rewards agents for moving forward at speeds exceeding \( v_u \), penalizes speeds between \( v_l \) and \( v_u \), and provides a moderate penalty for very low speeds. Agents failing to overcome the RO problem may settle for maintaining low speeds to avoid the harshest penalty.

\begin{figure}[t]
  \centering
   % \vspace{-10pt}
  \includegraphics[width=\linewidth]{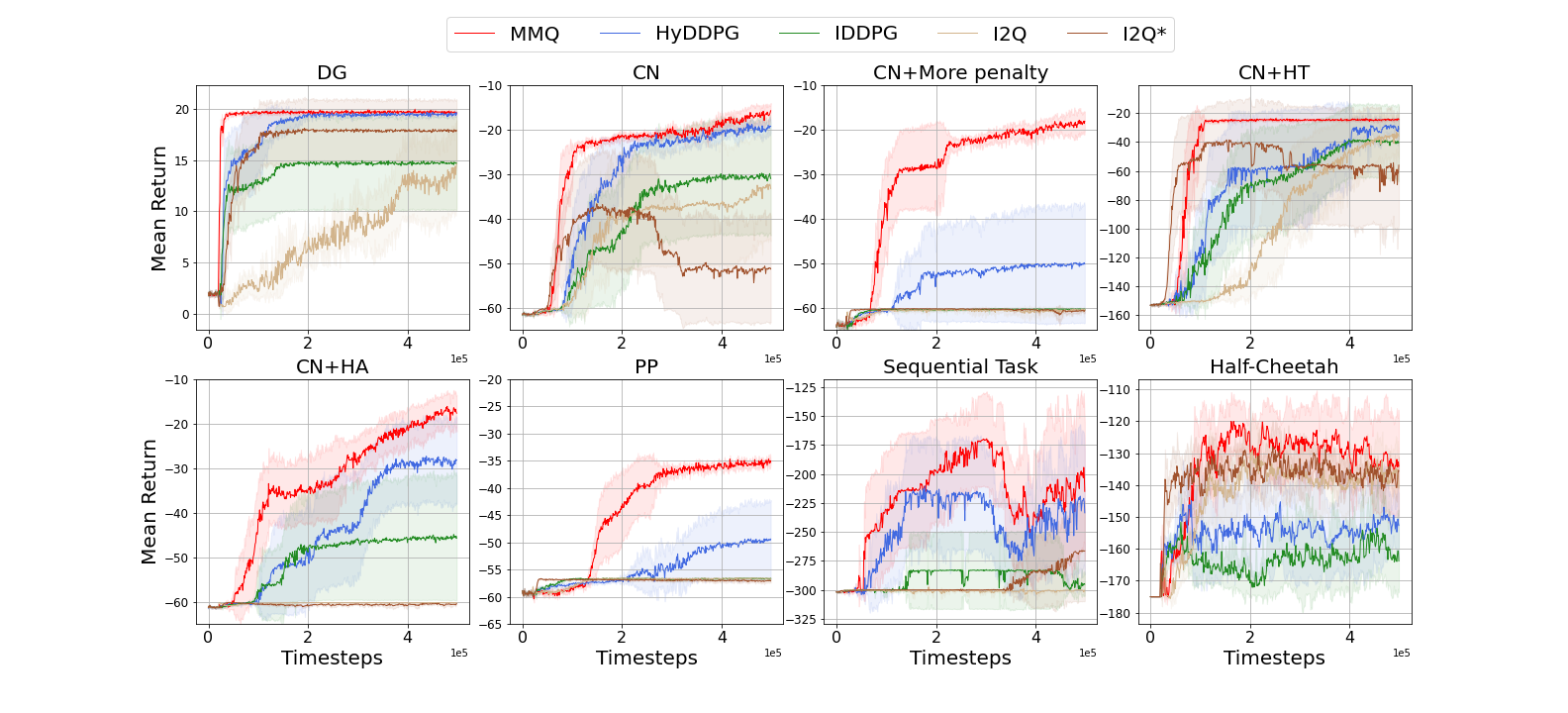}
  \caption{Performance comparison for two-agents setting in DG, MPE scenarios and Half-Cheetah}
  \label{fig:all_per}
\end{figure}

\begin{table*}[!ht]
\caption{Settings with $N=2$ agents. Mean Returns and $95\%$ Confidence Interval (over eight seeds) for all algorithms at the end of training. Values are bolded if their confidence intervals overlap with the maximum value. } 
\centering
\scalebox{0.8}{%
\begin{tabular}{p{2.3cm}p{2.8cm}p{2.5cm}p{2.3cm}p{2.5cm}p{2.3cm}p{2.5cm}}
\hline
Setting &Setting &MMQ &IDDPG &HyDDPG &I2Q* &I2Q\\
\hline
DG &$N=2$ &\bf{19.55$\pm$0.16}  &14.67$\pm$4.61  &\bf{19.47$\pm$0.17}  &\bf{17.84$\pm$3.01} & $14.62\pm$4.23 \\
\hline
\multirow{2}{=}{MPE Tasks} 
&CN & \bf{-15.66$\pm$1.75}  & -30.91$\pm$12.78 & \bf{-19.12$\pm$1.56} & -51.24$\pm$12.01 & -33.16$\pm$11.82 \\
&CN+more penalty & \bf{-18.01$\pm$2.24}  & -60.25$\pm$0.07 & -50.12$\pm$13.24 & -60.55$\pm$0.36 & -60.54$\pm$0.70 \\
&CN+HT & \bf{-24.25$\pm$0.49}  & \bf{-40.34$\pm$26.01} & \bf{-30.93$\pm$6.92} & \bf{-56.04$\pm$40.80} & \bf{-35.95$\pm$15.83} \\
&CN+HA & \bf{-17.63$\pm$4.12}  & -45.76$\pm$13.99 & \bf{-28.21$\pm$9.71} & -60.49$\pm$0.15 &-60.25$\pm$0.03 \\
&PP & \bf{-35.28$\pm$0.40}  & -56.63$\pm$0.11 & -49.40$\pm$7.18 & -57.10$\pm$0.33 & -56.67$\pm$0.14 \\
&Sequential Task  & \bf{-215.09$\pm$59.97} & -295.31$\pm$9.75 & \bf{-233.55$\pm$53.21} & -300.53$\pm$0.45 & \bf{-266.42$\pm$43.37} \\
\hline
Half-Cheetah & $2\times3$  &\bf{-134.09$\pm$16.05}  & -163.81$\pm$9.61 & -152.66$\pm$10.32 & \bf{-135.65$\pm$4.60} & \bf{-140.94$\pm$8.55} \\
\hline
\end{tabular}}
\label{tab:performance}
\end{table*}

\begin{table*}[!ht]
\caption{Setting with more than $2$ agents. Mean Returns and $95\%$ Confidence Interval (over eight seeds) for all algorithms at the end of training. Values are bolded if their confidence intervals overlap with the maximum value.} 
\centering
\scalebox{0.8}{%
\begin{tabular}{p{2.3cm}p{2.6cm}p{2.5cm}p{2.3cm}p{2.3cm}p{2.3cm}p{2.3cm}}
\hline
Setting &Setting &MMQ &IDDPG &HyDDPG &I2Q* &I2Q\\
\hline
\multirow{2}{=}{DG} 
&$N=3$ &\bf{19.51$\pm$0.15}  &15.45$\pm$3.76 &\bf{16.17$\pm$4.11}  &15.95$\pm$3.78  & $8.19\pm$3.85 \\
&$N=4$ &\bf{20.09$\pm$0.10}  &14.09$\pm$4.22  &12.77$\pm$4.56  &\bf{16.31$\pm$4.28} & $12.59\pm$4.45 \\
&$N=5$ &\bf{20.44$\pm$0.11}  &13.94$\pm$4.27  &16.02$\pm$3.97  &\bf{16.91$\pm$3.30} & $2.98\pm$0.70 \\
\hline
\multirow{2}{=}{MPE Tasks} 
&CN (N=3) & \bf{-34.79$\pm$2.75}  & -60.11$\pm$17.06 & \bf{-45.47$\pm$15.48} & \bf{-50.86$\pm$17.68} & -54.89$\pm$17.49 \\
&PP (N=3) & \bf{-37.86$\pm$1.01}  & -49.64$\pm$2.14 & -46.16$\pm$4.84 & -52.52$\pm$0.26 & -51.86$\pm$0.95 \\
\hline
\end{tabular}}
\label{tab:performance_sca}
\end{table*}

\begin{figure*}[!ht]
% \vspace{-10pt}
         \centering
     \begin{subfigure}[b]{0.7\textwidth}
         \includegraphics[width=\textwidth]{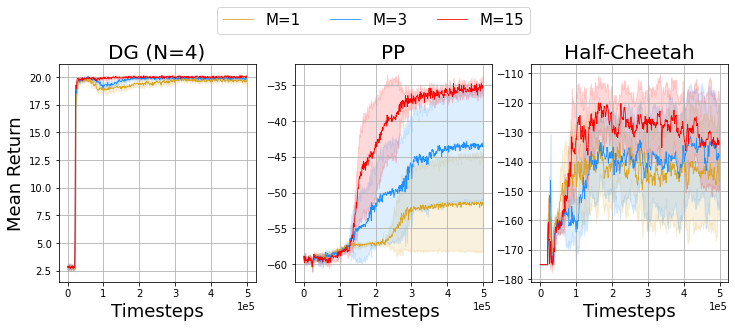}
         \caption{}
         \label{fig:abla}
     \end{subfigure}
     \begin{subfigure}[b]{0.25\textwidth}
         \centering
         \includegraphics[width=\textwidth]{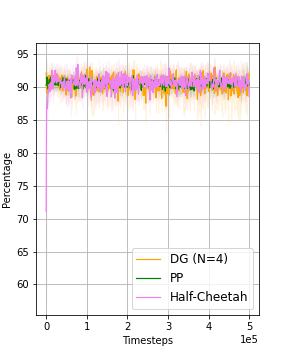}
         \caption{}
         \label{fig:pred_bound_dim}
     \end{subfigure} \\
    \caption{(a) Ablation study for different sample number $M$ in three tasks; (b) Percentage of each dim of true next states fall within the predicted quantile bound for three tasks}
\end{figure*}

\subsection{Baselines}
Our benchmarks include comparisons with three baseline algorithms: Ideal Independent Q-Learning (I2Q), Independent Deep Deterministic Policy Gradient (IDDPG), and Hysteretic DDPG (HyDDPG). We differentiate between two versions of I2Q: the original implementation by \cite{jiang2022i2q}, which we refer to as I2Q*, involves multiple updates of all network components after every 50 interaction steps. In contrast, our implementation, denoted as I2Q, updates only the critic network multiple times every 50 steps. This distinction allows us to more accurately assess the performance impact of these differing update strategies. 

\subsection{Experiment Results}
To compare the performance among different baselines, we consider the mean episode return over $8$ seeds, where the episode return is defined as the accumulated reward $\sum_{t=0}^T r_t$ for the whole episode, with $T$ being the episode length.

\paragraph{Differential Games} 

Our evaluations, depicted in Figure~\ref{fig:all_per} and Table~\ref{tab:performance}, show that MMQ outperforms other algorithms with 15 samples drawn from the quantile bounds predicted by two quantile models. HyDDPG and I2Q$^*$ also perform well. Interestingly, the performance of HyDDPG and IDDPG surpasses that reported for I2Q in \cite{jiang2022i2q}, possibly due to our implementation's emphasis on updating the critic network more frequently than the actor network to stabilize training. However, I2Q learns much slower compared to I2Q$^*$, which updates all modules rather than just the critic multiple times. These update strategies have different effects on various algorithms. With more agents, the performance of other algorithms degrades slightly, showing higher variability, as shown in Table~\ref{tab:performance_sca}. MMQ consistently identifies the optimal region in all test cases, demonstrating its higher sample efficiency and scalability. We also tested MMQ in a stochastic version of this game (Appendix~\ref{app:stochastic}), confirming that MMQ still outperforms the baselines under stochastic state transitions and reward dynamics.

\paragraph{Multiple Particle Environment}
As shown in Figure~\ref{fig:all_per} and Table~\ref{tab:performance}, our algorithm, using 15 samples from the predicted quantile bounds, successfully overcomes the Relative Over-generalisation (RO) problem in all settings. HyDDPG performs particularly well, especially in the lower penalty scenarios, such as \textbf{CN}, \textbf{CN+HT}, and \textbf{CN+HA}. In \textbf{CN} and \textbf{CN+HT}, I2Q$^*$ initially demonstrated some ability to solve the task but its performance deteriorated over time. We observed that the learned Q-values of some seeds increased rapidly and incorrectly simultaneously, despite setting the weight of the \(Q^{\mathrm{ss}}(s,s')\) value to a minimal parameter during the update process. This might be due to the challenges in computing effective \(Q^{\mathrm{ss}}(s,s')\) values in I2Q, leading to less accurate predictions of optimal states and resulting in cumulative estimation errors over time. With a higher solo-entry penalty in \textbf{CN+More Penalty}, all baselines' performance significantly declines, remaining stuck in suboptimal areas except for HyDDPG, which could learn to a limited extent. Our sampling-based approach, however, demonstrates robustness even with the increased RO problem. 

\textbf{PP} presents a more challenging task compared to the above settings due to its larger state space, more agents, and fast-moving prey that the predators must catch. Our results, illustrated in Figure~\ref{fig:all_per}, show our algorithm successfully overcomes the RO problem, demonstrating higher mean return and greater sample efficiency than other baselines. 

In the \textbf{Sequential Task}, both our MMQ and HyDDPG initially learned to enter target \(A\) directly. After discovering that entering target \(B\) could yield a bonus, they deliberately targeted \(B\), which led to a temporary decline in performance. MMQ recovered quicker than HyDDPG. I2Q began to show learning towards the end of the training while IDDPG failed in this task.

We also tested our algorithm in the \textbf{CN} and \textbf{PP} settings with an additional agent (as shown in Table~\ref{tab:performance_sca}). With one more agent, the state dimension increases and the transition dynamics become more complex. In \textbf{CN} with three agents, the results for HyDDPG and I2Q$^*$ showed large variability, indicating that these methods were effective only under certain initial conditions. In \textbf{PP} with three agents, which is more challenging, all baselines were stuck in suboptimal areas. However, our algorithm still managed to overcome the RO problem, demonstrating scalability with an increased number of agents. Further tests with the default reward setting (detailed in Appendix~\ref{app:dense}) show that our algorithm matches the performance of other baselines in settings without significant RO problems.

\paragraph{Multi-agent MuJoCo environment}
In this setting, we also utilized 15 samples drawn from the predicted quantile bounds. As depicted in Figure~\ref{fig:all_per}, MMQ remains competitive in this challenging environment, consistently surpassing other baselines during the latter half of the training period. Given the inherent complexity of the task, we maintained a mild level of the RO problem to preserve feasibility. This explains why I2Q was able to perform well, despite its sensitivity to RO problems noted in other environments. We provided results for one more partitions of Half-Cheetah (Half-Cheetah 4|2) and two partitions of Ant (Ant $2\times4$ and $4\times2$) with the same reward setting in Appendix \ref{app:more_Mamujoco}. In these tasks, MMQ consistently outperforming other baselines.

\paragraph{Ablation study: number of samples}

We conducted an ablation study to investigate the effect of varying the number of samples across three different environmental settings, as shown in Figure~\ref{fig:abla}. The results demonstrated that MMQ, even with just one sample, outperformed baselines in the DG and Half-Cheetah environments. In the PP environment, using one sample was slightly less effective than HyDDPG, but using three or more samples consistently outperformed all baselines. Additionally, increasing the number of samples \( M \) appeared to accelerate learning across the three settings, aligning with our theoretical analysis. Furthermore, we included a study (see Appendix~\ref{app:variation}) that employed a small ensemble of quantile models. Although this ensemble enlarged the predicted bounds and enhanced the percentage of true next states within these bounds, as shown in Figure~\ref{fig:ablaE_MeanDim}, it did not lead to further performance improvements. Thus, we did not incorporate the ensemble approach in the final results for simplicity.

To demonstrate the effectiveness of the quantile model, Figure~\ref{fig:pred_bound_dim} shows the percentage of each dimension of the true next states that fall within the predicted quantile bounds during the learning process for three environment settings. The percentage is calculated as follows: For each sample in the mini-batch, consisting of \( n \) samples, during the update process, we have the predicted bound \([l_d, u_d]\) for each dimension \( d \) of the state with value \( s_d \). The state has a total of \( D \) dimensions. If the state value \( s_d \) falls within the predicted bound, i.e., \( l_d \leq s_d \leq u_d \), we increment a count. If the total count across all dimensions and samples is \( P \), the percentage is then calculated as \(\frac{P}{n \cdot D} \times 100\%\).

The percentage is already high initially for the DG and CN environments but is a bit lower for Half-Cheetah, possibly due to its more complex dynamics compared to the other two. These results indicate that the quantile model can effectively capture most state changes as expected, suggesting that \( \hat{\mathcal{S}}_{s,a_i} \) closely reflects the true set \( \mathcal{S}_{s,a_i} \), as analyzed in Section~\ref{sec:mmq}.

\section{Discussion and Conclusions}

In this work, we introduced MaxMax Q-learning (MMQ), a novel algorithm designed to mitigate the Relative Over-generalisation problem in multi-agent collaborative tasks through the use of quantile models and optimistic sampling. Our key theoretical contribution establishes a connection between the accuracy of estimating the best next state and the convergence towards globally optimal joint policies within the MMQ framework, providing a solid foundation for understanding the algorithm's behaviour and performance.

Empirically, we demonstrated the effectiveness of MMQ across three diverse tasks, showing its ability to outperform existing baselines in terms of convergence speed, sample efficiency, and final performance. A notable outcome is MMQ’s scalability, as it can accommodate an increased number of agents while achieving superior performance with minimal samples, which reduces computational overhead—a crucial factor in practical multi-agent systems.

Despite these promising results, several limitations of MMQ remain. First, the Monte Carlo optimisation error in MMQ may exhibit exponential dependence as the dimensionality of the state space increases. However, this issue is somewhat mitigated in multi-agent settings, where the use of partially-observed MDPs limits the effective dimensionality. Second, although MMQ has demonstrated scalability in our experiments, challenges may arise in environments with very large populations of agents, where increased computational costs and coordination complexity could impact performance.

Looking ahead, future work may aim to address these limitations and further enhance MMQ’s capabilities. One promising direction involves developing adaptive mechanisms to gauge the informativeness of observations about other agents, which could improve the robustness of reward function learning, especially in scenarios with partial or noisy observations. Additionally, relaxing the current assumption of independence across dimensions in the quantile model by predicting the covariance matrix could capture more complex state dynamics, potentially leading to more accurate estimations. Further improvements in scalability might be achieved by reducing the number of samples required as the state space grows, thereby lowering computational demands in high-dimensional and large-agent scenarios. Together, these future developments aim to overcome current limitations and extend the applicability of MMQ to a broader range of multi-agent scenarios.

\subsection*{Acknowledgments}
\noindent
TZ acknowledges support from the UK Engineering and Physical
Sciences Research Council (EPSRC EP/W523793/1), through the
Statistics Centre for Doctoral Training at the University of Warwick.
\noindent
GM acknowledges support from a UKRI AI Turing Acceleration Fellowship (EPSRC EP/V024868/1).

\bibliography{main}

\newpage
\appendix
\onecolumn

\section{Environments: further details}
\label{app:env_d}

\textbf{Differential Games.} In the differential games \citep{jiang2022i2q}, \(N\) agents can move within the range \([-1,1]\). At each time step, agent \(i\) selects an action \(a_i \in [-1,1]\), and the state of agent \(i\) is updated to \(\operatorname{clip}\{x_i + 0.1 \times a_i, -1, 1\}\), where \(x_i\) is the previous state of agent \(i\). The \(\operatorname{clip}(y,y_{\mathrm{min}},y_{\mathrm{max}})\) function constrains \(y\) to the interval \([y_{\mathrm{min}},y_{\mathrm{max}}]\). The global state is represented by the position vector \((x_1, x_2)\). Agents receive rewards after each action, with the reward function for each time step defined as:
\begin{align*}
    r  =  
    \begin{cases*}
         a (\cos (l\pi / m) + 1) & if \(l \leq m\)  \\
         0 & if \(m < l \leq 0.6\) \\
         b (\cos (5\pi (l - 0.8)) + 1) & if \(0.6 < l \leq 1\) \\
         0 & if \(l > 1\)
    \end{cases*} 
\end{align*}
where \(a\) and \(b\) define the optimal and sub-optimal reward values, respectively; \(m\) determines the width of the zero-reward area between the sub-optimal and optimal reward zones. We set \(a = 0.5\), \(b = 0.15\), and \(m = 0.13(N-1)\). The location metric, \(l\), is calculated as:
\[
l = \sqrt{\frac{2}{N}\sum_{i=0}^{N}x_i^2}.
\]

As illustrated for \(N=2\) in Figure~\ref{fig:all_env}, there is a narrow optimal reward region at the center, surrounded by a wide zero-reward area. The edges also offer suboptimal rewards. This configuration makes it unlikely for decentralized agents to randomly converge on the optimal region, often resulting in agents getting stuck in the large suboptimal zones. This setup presents significant challenges for exploration and coordination. To assess these dynamics, agent positions are randomly initialized at the start of each 25-timestep episode. The combination of a narrow optimal reward zone and decentralized partial observability creates a challenging multi-agent exploration problem. Success in this environment requires effectively mitigating the RO issue to discover joint policies that maximize cumulative reward.

\textbf{Multiple Particle Environment:} The primary objective is for two disk-shaped agents, \(D_1\) and \(D_2\), with centers \(x_1\) and \(x_2\) and radius \(r_{\mathrm{a}}\), to simultaneously reach a disk-shaped target or moving prey with center \(x_{\mathrm{t}}\) and radius \(r_{\mathrm{t}}\). To foster coordination, we introduce a penalty for scenarios where only one agent is within a specific distance from the target. Specifically, we define a disk \(D\) centered on the target with radius \(r_D = r_{\mathrm{t}} + \alpha\) and penalize agents if only one is within \(D\). This setup is designed to illustrate the RO problem, where agents might prefer staying outside \(D\) rather than risking being the only one inside it. The task difficulty increases as \(\alpha\) decreases. For our experiments, we maintain a default value of \(r_{\mathrm{a}} = 0.15\), \(r_{\mathrm{t}} = 0.05\), and set \(\alpha=0.3\) for the Predator-Prey task and \(\alpha=0.2\) for other tasks. The reward function designed to reflect the RO problem is given by:
\[
r_{\mathrm{CN}} =
\begin{cases*}
R_{in} & if \(D_i \cap D \neq \emptyset\), \(i=1,2\) \\
R_{out} & if \((D_1 \cup D_2) \cap D = \emptyset\) \\
R_{out} - p & otherwise.
\end{cases*}
\]

Table \ref{tab:mpe_reward} outlines the specific reward parameters used across different settings, showcasing how the reward values are adjusted to enhance or mitigate the challenges posed by the RO problem. This table provides a clear overview of how penalties and rewards are configured to encourage effective coordination and exploration strategies among agents.

\begin{table*}[!h]
\caption{RO reward design for different scenarios in MPE task} 
\centering
\scalebox{0.8}{%
\begin{tabular}{p{4.2cm}p{3.2cm}p{3.2cm}p{2.3cm}}
\hline
Setting &$R_{out}$ &$R_{in}$ &$p$\\
\hline
CN & $-3(r_D + r_{\mathrm{a}})$ & $-3\min_{i=1,2} \| x_{\mathrm{t}} - x_i \|$ & 0.2\\
More Penalty & $-3(r_D + r_{\mathrm{a}})$ & $-3\min_{i=1,2} \| x_{\mathrm{t}} - x_i \|$ & $0.5$ \\
HT (Target A) &$-3$ &$0$ &$0.5$ \\
 HT (Target B) &$-3$ &$-2.5$ &$0$ \\
HA & $-3(r_D + r_{\mathrm{a}})$ & $-3\min_{i=1,2} \| x_{\mathrm{t}} - x_i \|$ & $0.2$\\
Predator-Prey &$-3(r_D + r_{\mathrm{a}})$ & $-3\min_{i=1,2} \| x_{\mathrm{t}} - x_i \|$ & $0.5$\\
Sequential (without cargo) &$-6$ &$-3$ &$0.5$\\
Sequential (with cargo) &$-6$ &$-0.5$ &$0.5$ \\
\hline
\end{tabular}}
\label{tab:mpe_reward}
\end{table*}

\newpage
\section{Proof of Theorems \ref{conv} and \ref{distance}}
\label{sec:proofs}

We first prove a theorem and lemma before moving on to the proof of Theorem~\ref{conv}.

\begin{theorem} \label{theorem_1}
Let $\mathcal{T}_{\hat{S}}$ be the operator defined by $\mathcal{T}_{\hat{S}}Q(s, a_i) = \max_{s' \in \hat{S}} (R(s, s') + \gamma \max_{a'_i}Q(s', a'_i))$, then $\mathcal{T}_{\hat{S}}$ is a contraction in sup-norm.
\end{theorem}
\begin{proof}
Let $Q_1$ and $Q_2$ be two Q-functions, i.e., real-valued functions on $\mathcal{S} \times \mathcal{A}_i$. By definition of the operator $\mathcal{T}_{\hat{S}}$ and of the argmax, we can rewrite the sup-norm $\| \mathcal{T}_{\hat{S}}Q_1 - \mathcal{T}_{\hat{S}}Q_2\|_{\infty}$ as
\begin{align}
    \| \mathcal{T}_{\hat{S}}Q_1 - \mathcal{T}_{\hat{S}}Q_2\|_{\infty} 
    &= \max_{s,a_i}\left| \max_{s'\in \hat{S}} \Big(R(s,s')+\gamma \max_{a'_i}Q_1(s',a'_i) \Big) - \max_{s'\in \hat{S}} \Big(R(s,s')+\gamma \max_{a'_i}Q_2(s',a'_i) \Big) \right| \\ 
    &= \max_{s,a_i}\left| (R(s,s'_1)+\gamma\max_{a'_i}Q_1(s'_1,a'_i)) - (R(s,s'_2)+\gamma\max_{a'_i}Q_2(s'_2,a'_i))\right|,
\end{align}
where $s'_j=\arg\max_{s'\in \hat{S}}(R(s,s')+\gamma\max_{a'_i}Q_j(s',a'_i))$, $j=1,2$. It follows that
\begin{align}
    \| \mathcal{T}_{\hat{S}}Q_1 - \mathcal{T}_{\hat{S}}Q_2\|_{\infty}
    & = \max_{s,a_i}\left| R(s,s'_1) - R(s,s'_2)+\gamma\max_{a'_i}Q_1(s'_1,a'_i) - \gamma\max_{a'_i}Q_2(s'_2,a'_i))\right| \\
    & \leq \max_{s,a_i}\left| R(s,s'_1) - R(s,s'_2) \right| +\gamma \max_{s,a_i}\left|\max_{a'_i}Q_1(s'_1,a'_i) - \max_{a'_i}Q_2(s'_2,a'_i))\right| \\
    &\leq \gamma  \max_{s,a_i}\left|\max_{a'_i}Q_1(s'_1,a'_i) - \max_{a'_i}Q_2(s'_2,a'_i))\right| \\
    & \leq \gamma\max_{s,a'_i} \left| Q_1(s, a'_i) - Q_2(s, a'_i) \right| \\
    & = \gamma \|Q_1-Q_2\|_{\infty},
\end{align}
which concludes the proof of the theorem.
\end{proof}

Therefore, according to the contraction mapping theorem, $\mathcal{T}_{\hat{S}}$ converges to a unique fixed point which we denote by $Q$. Additionally, we denote by $Q^*$ the optimal Q-function, which is the fixed point of $\mathcal{T}_{\mathcal{S}}$.

We consider the following assumptions:
\begin{enumerate}[label=A.\arabic*, noitemsep]
    \item \label{as:Rlip} $R(s, \cdot)$ is Lipschitz continuous for any $s \in \mathcal{S}$
    \item $Q^*$ is twice continuously differentiable function and Lipschitz continuous
    \item \label{as:order2} For any $s \in \mathcal{S}$ and any maximizer $a^*$ of $Q^*(s, \cdot)$, it holds that $a^*$ is a maximizer of order 2, i.e., there are constants $c$ and $\delta$ such that
    \[
    Q^*(s, a^*) - Q^*(s, a) \geq c \| a^* - a \|, \qquad \forall a \in \mathcal{A}_i, \| a^* - a \| \leq \delta.
    \]
\end{enumerate}

\begin{lemma} Define $q^*(s, s') = R(s, s') + \gamma \max_{a'_i} Q^*(s', a'_i)$ and consider Assumptions~\ref{as:Rlip}-\ref{as:order2}. If $d(s'^{*}, \hat{s}'^{*}) \leq \epsilon$, then there exists a constant $K$ such that $|q^*(s, s'^{*}) - q^*(s, \hat{s}'^{*})| \leq K \epsilon$ and where $s'^{*}, \hat{s}'^{*}$ are the best next state in $\mathcal{S}_{s, a_i}$ and $\hat{S}$, respectively.
\end{lemma}

\begin{proof}
Based on the definition of $q^*$ and on the triangle inequality, it holds that
\begin{equation}
\begin{aligned}
    |q^*(s, s'^{*}) - q^*(s, \hat{s}'^{*})|
    &= |R(s, s'^{*}) + \gamma \max_{a'_i}Q^*(s'^{*}, a'_i) - R(s, \hat{s}'^{*}) - \gamma\max_{a'_i}Q^*(\hat{s}'^{*}, a'_i)|\\
    &\leq |R(s, s'^{*}) - R(s, \hat{s}'^{*})| + \gamma|\max_{a'_i} Q^*(s'^{*}, a'_i) - \max_{a'_i} Q^*(\hat{s}'^{*}, a'_i)| \\
    & \leq \epsilon K_R  + \gamma|\max_{a'_i} Q^*(s'^{*}, a'_i) - \max_{a'_i} Q^*(\hat{s}'^{*}, a'_i)|,
\end{aligned}
\end{equation}
with $K_R$ the Lipschitz constant for $R(s, \cdot)$. Let $a'^{*}_i = \arg\max_{a'_i} Q^*(s'^{*}, a'_i)$ and $\hat{a}'^{*}_i = \arg\max_{a'_i} Q^*(\hat{s}'^{*}, a'_i)$. According to Theorem 6.2 in \citep{still2018lectures}, there are $\epsilon, L > 0$ such that for all $s' \in \mathcal{S}$ verifying $\| s' - s'^* \| \leq \epsilon$, there exists a local maximizer $\tilde{a}'_i$ verifying
\[
\| \tilde{a}'_i - a'^{*}_i \| \leq L \| s' - s'^{*} \|.
\]
It follows that $\| (s'^{*}, a'^*_i) - (\hat{s}'^{*}, \hat{a}'^*_i) \|^2 \leq \epsilon^2 (1 + L^2)$, and we finally obtain that
\[
|q^*(s, s'^{*}) - q^*(s, \hat{s}'^{*})| \leq \epsilon K_R + \epsilon \sqrt{1 + L^2} K_Q = K \epsilon,
\]
with $K_Q$ the Lipschitz constant for $Q$ and with $K = K_R + \sqrt{1 + L^2}K_Q$.
\end{proof}

\begin{proof}[Proof of Theorem~\ref{conv}]
We now need to consider the accumulation of error between time steps and therefore introduce an additional subscript $t$ on all the time-varying quantities. In addition, let $q(s, s') = R(s, s') + \gamma \max_{a'_i}Q(s', a'_i)$ be the function related to the approximate Q-function $Q$ and define $\Delta_{t+1} = |q^*(s_t, \hat{s}'^{*}_{t}) - q(s_t, \hat{s}'^{*}_{t})|$. Then, we have
\begin{equation}\label{recursion_1}
\begin{aligned}
|q^*(s_t, s'^{*}_{t}) - q(s_t, \hat{s}'^{*}_{t})|
& =|q^*(s_t, s'^{*}_{t})  - q^*(s_t, \hat{s}'^{*}_{t}) + q^*(s_t,  \hat{s}'^{*}_{t})  - q(s_t, \hat{s}'^{*}_{t}) | \\
& \leq K \epsilon + \Delta_{t+1}.
\end{aligned}
\end{equation}
We can manipulate the expression of $\Delta_{t+1}$ to make the left hand side of \eqref{recursion_1} appear, but at time $t+1$. First, we observe that
\begin{align*}
    \Delta_{t+1} & = |q^*(s_t, \hat{s}'^{*}_{t}) - q(s_t,  \hat{s}'^{*}_{t})| \\
    & = |R(s_t, \hat{s}'^{*}_{t}) + \gamma \max_{a'_{i,t+1}} Q^*(\hat{s}'^{*}_{t}, a'_{i,t+1}) - R(s_t, \hat{s}'^{*}_{t}) - \gamma \max_{a'_{i,t+1}}Q(\hat{s}'^{*}_{t}, a'_{i,t+1})|\\
    & = \gamma\Big|\max_{a'_{i,t+1}} Q^*(\hat{s}'^{*}_{t}, a'_{i,t+1}) - \max_{a'_{i,t+1}} Q(\hat{s}'^{*}_{t}, a'_{i,t+1})\Big|.
\end{align*}
The first term will have the optimal action at state $\hat{s}'^{*}_{t}$ as maximizer, and we denote $\hat{s}'^{*}_{t+1}$ the induced state. Similarly, we denote by $s'^{*}_{t+1}$ the state following from the optimal action under $Q$. It follows that
\begin{equation}\label{recursion_2}
\begin{aligned}
    \Delta_{t+1} & = \gamma\Big|R(\hat{s}'^{*}_{t},\hat{s}'^{*}_{t+1}) + \gamma \max_{a'_{i,t+2}} Q^*(\hat{s}'^{*}_{t+1}, a'_{i,t+2}) - R(\hat{s}'^{*}_{t},s'^*_{t+1}) - \gamma \max_{a'_{i,t+2}} Q(s'^{*}_{t+1}, a'_{i,t+2})\Big| \\
    & = \gamma \big|q^*(\hat{s}'^{*}_{t}, s'^{*}_{t+1}) - q(\hat{s}'^{*}_{t}, \hat{s}'^{*}_{t+1}) \big|\\
    & \leq \gamma K \epsilon + \gamma  \Delta_{t+2}.
\end{aligned}
\end{equation}
Based on \eqref{recursion_1} and \eqref{recursion_2}, we have:
\[
    \big|q^*(s_t, s'^{*}_{t}) - q(s_t, \hat{s}'^{*}_{t})\big| \leq  (1+\gamma^2+\cdots+\gamma^{T-1}) K \epsilon + \gamma^{T-1} \Delta_{t+T}   
\]
As $T$ goes to the infinity, we have
\[
    \big|q^*(s_t, s'^{*}_{t}) - q(s_t, \hat{s}'^{*}_{t})\big| = \frac{1}{1-\gamma}K\epsilon
\]
Since it holds that $|Q^*(s_t, a_{i,t}) - Q(s_t, a_{i,t})|= |q^*(s_t, s^{*}_{t+1}) - q(s_t, \hat{s}'^{*}_{t+1})|$, we have
\[
|Q^*(s_t, a_{i,t}) - Q(s_t, \hat{a}_{i,t})| \leq \frac{1}{1-\gamma}K\epsilon
\]
which concludes the proof of the theorem.
\end{proof}

\begin{proof}[Proof of Theorem~\ref{distance}]
Consider the random variables $Y_k = |  s'_k - \tilde{s}^{\prime *}|$ for all $k \in \{ 
1,\dots, M \}$. For simplicity, we used $c$ to denote $\tilde{s}^{\prime *}$ in the following. Their survival function is found to be
\begin{align*}
S_Y(y) & = P(|s'_k-c| > y ) \\ & = P(s'_k > c + y) + P(s'_k < c-y) \\
& =  \begin{cases*} \begin{aligned}  & 1, & y \leq 0 \\ & 1-\frac{y}{u},  & 0<y\leq u-|c| \\ & \frac{u + |c| - y}{2u}, & u-|c| < y < u + |c| \\&  0, & y \geq u + |c| \end{aligned} 
\end{cases*} 
\end{align*}
The survival function of $\min_{k = 1, \cdots, M} Y_k$ is $S_Y(y)^M$ and 
\begin{align*}
\mathbb{E}\Big[\min_{k=1,\dots,M} Y_k\Big] 
& = \int_{0}^u S_Y(y)^M dy \\ 
& = \int_0^{u-|c|} \left(1 - \frac{y}{u}\right)^M dy + \int_{u-|c|}^{u+|c|}\left(\frac{u+|c|-y}{2u}\right)^Mdy\\
& = \frac{-u}{M+1}\left(1-\frac{y}{u}\right)^{M+1}\bigg |_{0}^{u-|c|} + \frac{-2u}{M+1}\left(\frac{u + |c|-y}{2u}\right)^{M+1}\bigg |_{u-|c|}^{u+|c|} \\
& = \frac{u}{M+1}\left(1 + \frac{|c|^{M+1}}{u^{M+1}}\right)\\
& < \frac{2u}{M + 1},
\end{align*}
which concludes the proof of the theorem.
\end{proof}

\section{Ablation Study}
\subsection{Negative Reward Shifting}
\label{app:negative}

To enhance exploration and mitigate overestimation in our MMQ approach, we have integrated a negative linear reward shifting technique \citep{sun2022exploit}. This method involves a downward adjustment of the reward function (by subtracting a constant from all rewards) during training. Analogous to starting with a higher initial value in the neural value function, this adjustment assigns inflated Q-values to less frequently visited state-action pairs, thus prioritizing them during the learning process. Our implementation of this insight prioritizes optimal next states, which are initially less common than suboptimal ones. By initializing the value function at a higher level, these optimal states receive more favorable evaluations. This technique proves particularly beneficial in our double-max-style updates, as it aids in the accurate selection of next states for updates and counteracts the overestimation of Q-values associated with chosen actions.

We conducted an ablation study on negative reward shifting in the differential games with \(N=2\) and \(N=4\), depicted in Figure \ref{fig:neg}. The results indicated that negative reward shifting significantly enhanced the performance of our algorithm. Additionally, we visualized the evolution of the learned Q-values during the initial 16 updates of training in Figure \ref{fig:Q_NC}. With the same Q-value initialization, the central part of the state space, which represents infrequently achieved states, exhibited higher values than the edge parts under the negative reward shifting regimen. In contrast, when utilizing the standard positive reward setting, the central part displayed lower values compared to the edges initially, underscoring the effectiveness of this reward adjustment strategy in promoting a more balanced exploration and accurate value estimation.

\begin{figure*}[!ht]
         \centering
     \begin{subfigure}[b]{0.49\textwidth}
         \includegraphics[width=\textwidth]{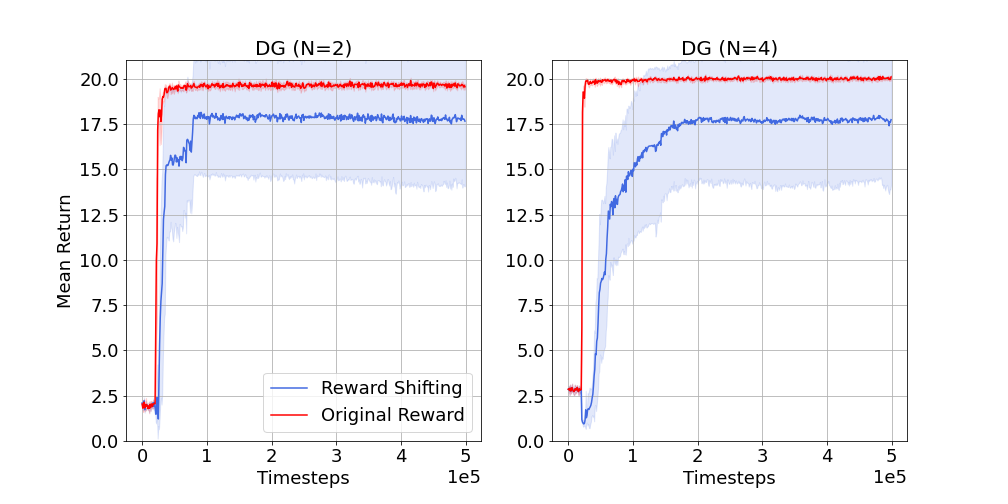}
         \caption{MMQ with Quantile model}
         \label{fig:neg_Q}
     \end{subfigure}
     \begin{subfigure}[b]{0.49\textwidth}
         \centering
         \includegraphics[width=\textwidth]{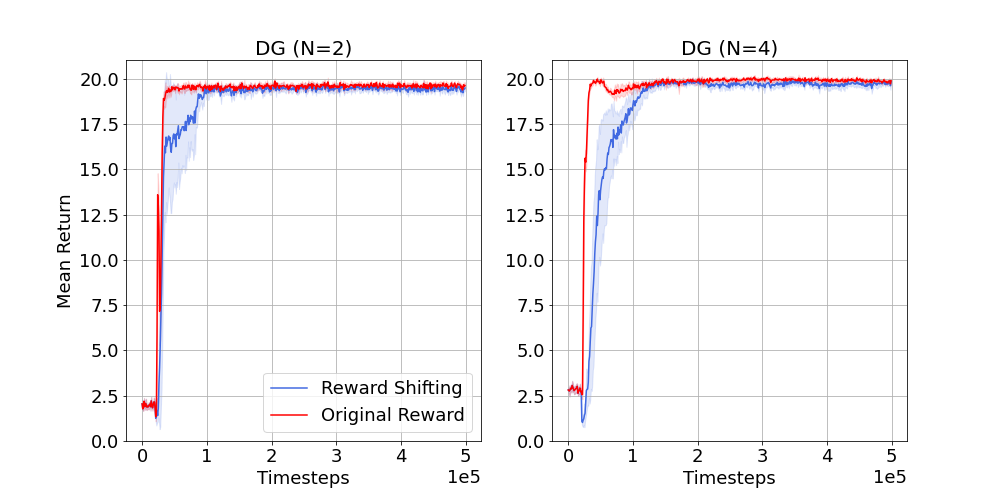}
         \caption{MMQ with Gaussian model}
         \label{fig:neg_G}
     \end{subfigure} 
     \caption{Performance comparison between reward shifting and original reward}
     \label{fig:neg}
\end{figure*}

\begin{figure}[!h]
  \centering
\includegraphics[width=0.8\linewidth]{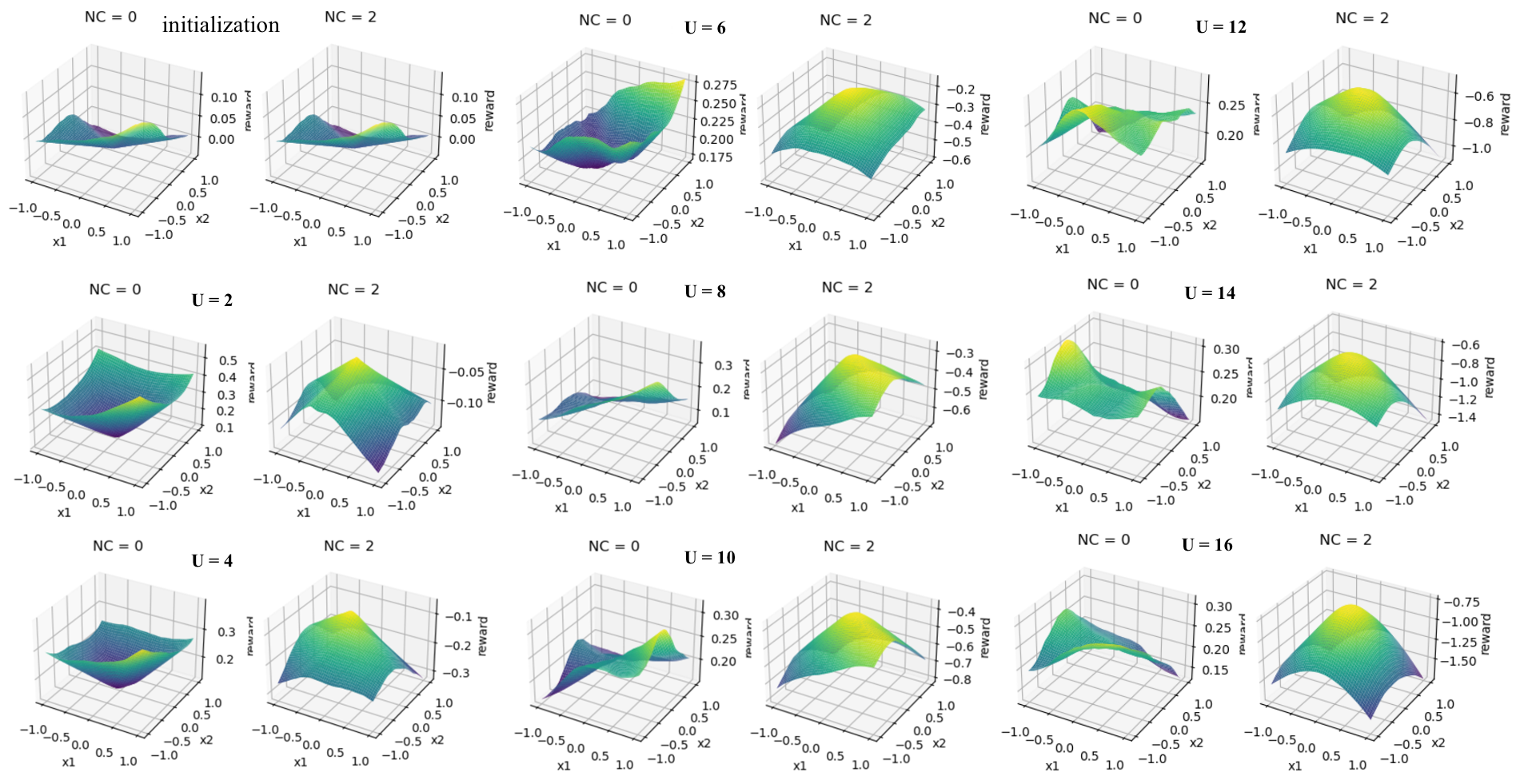}
  \caption{Comparison of Q-value for our algorithm with and without Negative reward shifting; NC=0, without negative reward shifting; NC=2, with negative reward shifting and the shifting constant = 2.}
  \label{fig:Q_NC}
\end{figure} 

\newpage
\subsection{Ensemble}
\label{app:variation}

In the MMQ algorithm, each agent uses two quantile models to capture variations in environment dynamics. While a single model on each side of the quantile bound provides valuable insights, its capacity to fully encapsulate the range of possible outcomes for state-action pairs can be limited. This limitation stems from the evolving nature of the agents' policies and the inherent uncertainties within the environment.

To enhance the robustness of our model and better address these challenges, we experimented with an ensemble approach. We utilized \( K \) pairs of models, denoted as \( \{g^{\tau_l}_{i,k}(s' | s, a_i; \phi^l_{k,i})\}_{k=1}^K \) and \( \{g^{\tau_u}_{i,k}(s' | s, a_i; \phi^u_{k,i})\}_{k=1}^K \), where each model in the ensemble is characterized by its own unique parameter set \( \phi_{k,i} \). This ensemble configuration allows for a broader representation of the dynamics, providing a more diverse set of potential outcomes. We define \( \mathcal{S}^*_{s, a_i} \) as the union of outcomes from all models: \( \mathcal{S}^*_{s, a_i} = \bigcup_{k=1}^K \{s'_{k, 1},\dots, s'_{k, M}\} \), where \( s'_{k, 1},\dots, s'_{k, M} \) are independent and identically distributed samples drawn from the bounds predicted by each pair of quantile models.

During the initial learning phases, these \( K \) models yield a broader range of predictions, as depicted in Fig.\ref{fig:ablaE_MeanDim}, reflecting different potential outcomes for the same state-action pairs. However, this approach did not lead to further improvement in performance, as a single quantile network was already sufficient to predict the bounds accurately. Consequently, this ensemble method did not enhance our algorithm's effectiveness in practice.

\begin{figure*}[!ht]
         \centering
     \begin{subfigure}[b]{0.65\textwidth}
         \includegraphics[width=\textwidth]{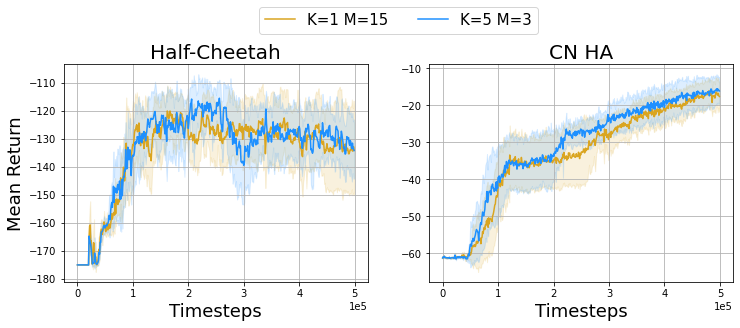}
         \caption{Ablation study for Ensemble}
         \label{fig:Abla_E}
     \end{subfigure}
     \hspace{0.02\textwidth}
     \begin{subfigure}[b]{0.3\textwidth}
         \centering
         \includegraphics[width=.8\textwidth]{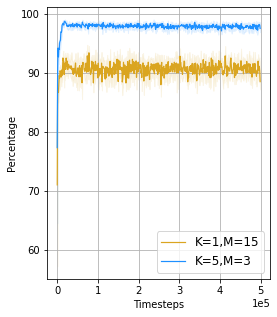}
         \caption{Percentage of true next states fall within the predicted quantile bound (each dim)}
         \label{fig:ablaE_MeanDim}
     \end{subfigure} 
\end{figure*}

\subsection{Comparison between quantile model and Gaussian model }\label{Quantile_Gau}
For the Gaussian model, the neural network parameters, are learned by minimizing a loss function based on the negative log-likelihood of the Gaussian distribution using experiences from its replay buffer $\mathcal{D}_i$:
\begin{align}
L(\phi_i) = -\mathbb{E}_{s,a_i\sim\mathcal{D}_i}\left[ \sum_{d=1}^D \log \left( \mathcal{N}(s'_d | \mu_{\phi_i, d}(s, a_i), \sigma_{\phi_i, d}^2(s, a_i)) \right)\right].
\label{eq:F_loss}
\end{align}
Here, \( \mu_{\phi_i, d}(s, a_i) \) and \( \sigma_{\phi_i, d}^2(s, a_i) \) are the neural network's predictions of the mean and variance for each dimension \(d\) of the next state \(s'\). The loss function sums the log-likelihoods across all dimensions \(D\) of the state space, focusing on accurately modeling the distribution of each state dimension. As shown in \ref{fig:pred_var}, the Gaussian model could also capture the other agents' influence on the observed state changes. The result in \ref{fig:com_G_Q} showed MMQ with the quantile model is slightly better than MMQ with the Gaussian model.

\begin{figure*}[!h]
         \centering
     \begin{subfigure}[b]{0.65\textwidth}
         \includegraphics[width=\textwidth]{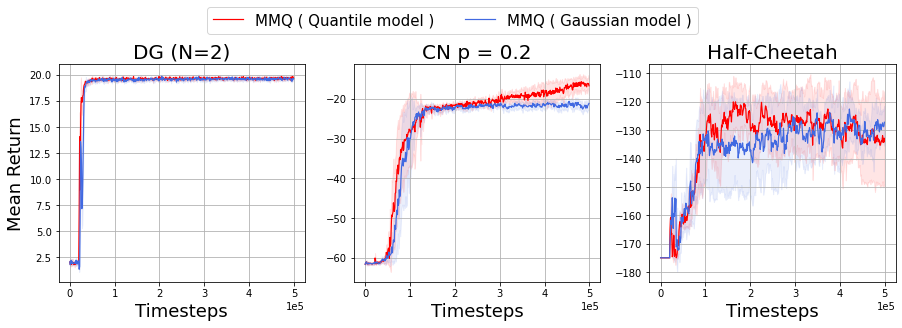}
         \caption{Comparison between Gaussian model and Quantile model}
         \label{fig:com_G_Q}
     \end{subfigure}
     \hspace{0.02\textwidth}
     \begin{subfigure}[b]{0.3\textwidth}
         \centering
         \includegraphics[width=.8\textwidth]{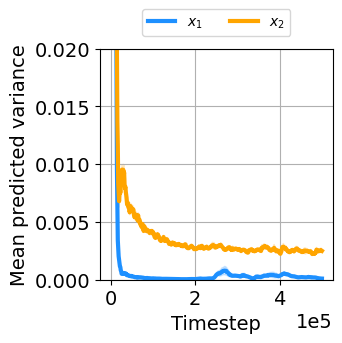}
         \caption{Predicted Variance of Agent 1's Forward Model in Differential Games DG (N=2)}
         \label{fig:pred_var}
     \end{subfigure} 
\end{figure*}

\section{Additional results}

\subsection{Learning curves in stochastic environment}
\label{app:stochastic}
To assess the robustness of our algorithm in the stochastic environment, we add Gaussian noise to the transition and the reward function. For stochastic state transition, we add a Gaussian noise to the position update of each agent: $x_i=\clip\{x_i + 0.1 \times a_i + z, -1, 1\}$, while $z\sim\mathcal{N}(0,\sigma^2)$. Here we test in two levels of stochasticity: $\sigma_s=0.02$ and $\sigma_s=0.05$. For the stochastic reward, we add Gaussian noise to the reward, $z\sim\mathcal{N}(0,\sigma^2)$, and the reward becomes $r = r + z$. We also test in two levels of noise: $\sigma_r=0.05$ and $\sigma_r=0.1$. In Figure \ref{fig:DG_noisy}, with noise in state transition, our algorithm still perform the best compared to other baselines. Although we use a deterministic function as reward function, our algorithm still find the global optimum for all seeds when there exits reward noise whereas the performance of other baselines degraded.

\begin{figure}[!ht]
  \centering
  \includegraphics[width=0.8\linewidth]{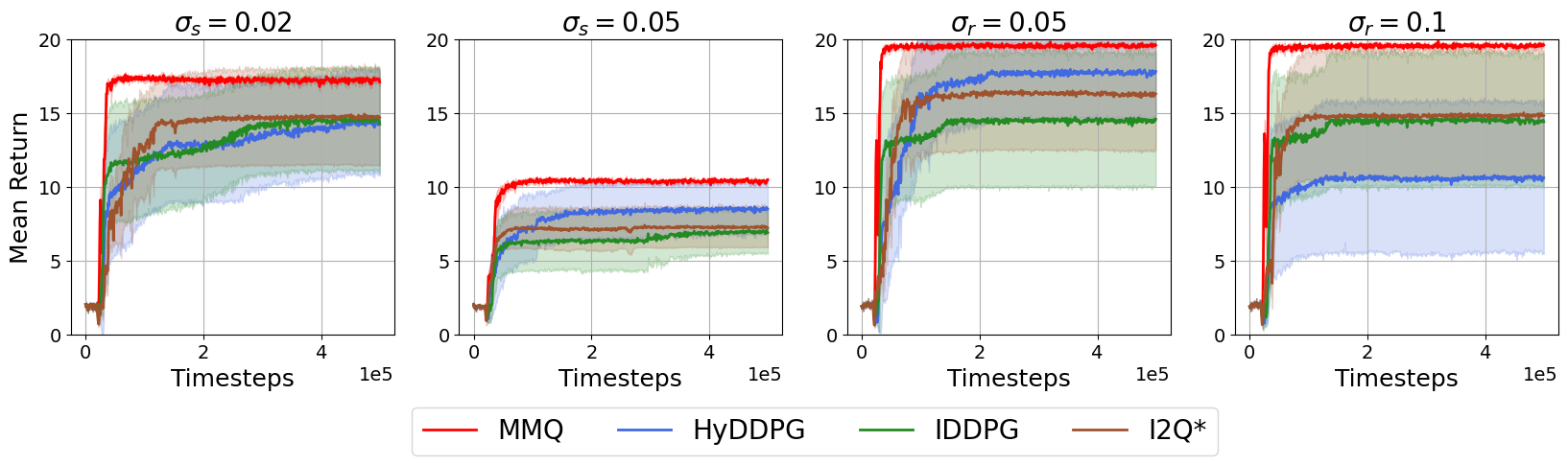}
  \caption{Learning curves in stochastic version of differential game (N=2)}
  \label{fig:DG_noisy}
\end{figure}

\subsection{Default Reward setting}
\label{app:dense}
We test our algorithm in cooperative navigation and Predator-Prey with the default reward setting \citep{lowe2017multi}. In the cooperative navigation, the result showed with a learned reward function, our algorithm learned a bit slower than two model-free baselines, HyDDPG and IDDPG, but finally achieved a similar level. When using the true reward from experiences (we then use $r(s')$ rather than $r(\hat{s}')$), our algorithm could achieve a similar performance as two baselines. Therefore, we assumed the reward estimation error in dense reward settings may impact our algorithm's performance slightly. In Predator-Prey with default sparse reward that agents would get $+10$ only when they collided with the prey, our algorithm outperforms all baselines.

\begin{figure}[!h]
  \centering
  \includegraphics[width=\linewidth]{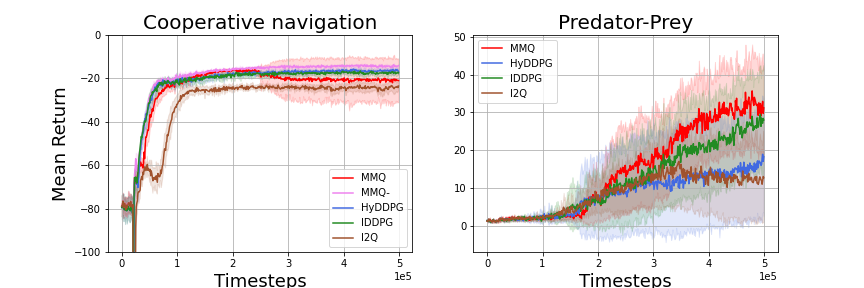}
  \caption{Cooperative navigation and Predator-Prey environment with default reward; MMQ$^-$ represent that we replace the learned reward function using the true reward from the experience; }
  \label{fig:oringal_Reward}
\end{figure}

\subsection{Additional Results for MAmujoco}
\label{app:more_Mamujoco}

We add results for three scenarios in Figure \ref{fig:HC42}: (1) Half-Cheetah $4|2$, there are two agents, one agent control $4$ joints and the other agent controls $2$ joints. (2) Ant $2\times4$, two agents, each agent control $4$ joints. (3) Ant $4\times2$, four agents and each agent control $2$ joints. MMQ slightly outperforms all other baselines. 

\begin{figure}[!h]
  \centering
  \includegraphics[width=0.8\linewidth]{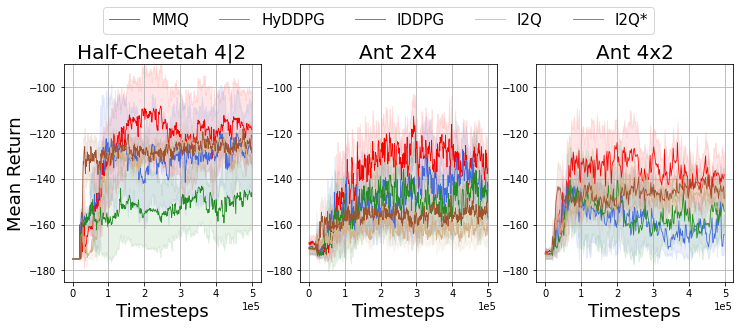}
  \caption{Learning curves on Multi-Agent Mujoco with RO reward setting}
  \label{fig:HC42}
\end{figure}

\newpage
\section{Hyperparameters}

\begin{table}[!ht]
\caption{Common Model and training hyperparameters used in all algorithms}
\label{tab:hyperparameters}
\centering
\scalebox{1}{%
\begin{tabular}{p{5cm}p{5cm}}
\hline
 & Parameter \\
\hline
Model architecture &FC layers[256,256] \\
Forward model architecture &FC layers[256,256] \\
Replay Buffer size &550000 \\
Batch size &100 \\
Optimizer &Adam \\
Learning rate &0.001 \\
Discount factor gamma &0.99 \\
Network initialization &Xavier \\
Activation & ReLU \\
Number of seeds & 8 \\
Total environment step &500000 \\
Exploration epsilon &0.1 \\
Episode length &25 Steps (expect 50 Steps for MPE sequential task) \\
Experience collection steps before training $T_{pretrain}$ &20000 \\
\hline
\end{tabular}}
\end{table}

The hyperparameters used in our experiments are listed in Table \ref{tab:hyperparameters}. All experiments were conducted on an NVIDIA TITAN V GPU. The training for each task typically completed within 5 hours.

\end{document}